\documentclass[11pt]{article}
\newcommand{\PublicationLicenseNotice}{%
  Submitted to arXiv under the arXiv.org perpetual, non-exclusive license 1.0.%
}

\usepackage[margin=1in]{geometry}
\usepackage{amsmath,amssymb,amsthm,mathtools,tabularx}
\usepackage{booktabs,longtable,array,enumitem,adjustbox}
\usepackage{xcolor}
\usepackage{hyperref}
\usepackage{listings}
\usepackage{microtype}
\usepackage{tikz}
\usepackage{stmaryrd}

\usetikzlibrary{positioning, arrows.meta, calc}
\definecolor{enclavebg}{HTML}{F0F4FF}
\definecolor{enclavedraw}{HTML}{2563EB}
\definecolor{teebg}{HTML}{F0FDF4}
\definecolor{teedraw}{HTML}{16A34A}
\definecolor{cohomologybg}{HTML}{FAF5FF}
\definecolor{cohomologydraw}{HTML}{9333EA}
\definecolor{commitbg}{HTML}{ECFDF5}
\definecolor{commitdraw}{HTML}{059669}
\definecolor{quarantinebg}{HTML}{FEF2F2}
\definecolor{quarantinedraw}{HTML}{DC2626}
\definecolor{textcolor}{HTML}{0F172A}

\tikzset{
  enclave/.style={
    rectangle,
    draw=enclavedraw,
    fill=enclavebg,
    rounded corners=6pt,
    line width=1.2pt,
    text width=8.6cm,
    align=center,
    inner sep=10pt,
    font=\sffamily\color{textcolor}
  },
  tee/.style={
    rectangle,
    draw=teedraw,
    fill=teebg,
    rounded corners=6pt,
    line width=1.2pt,
    text width=9.4cm,
    align=left,
    inner sep=10pt,
    font=\sffamily\color{textcolor}
  },
  cohomology/.style={
    rectangle,
    draw=cohomologydraw,
    fill=cohomologybg,
    rounded corners=6pt,
    line width=1.2pt,
    text width=9.4cm,
    align=left,
    inner sep=10pt,
    font=\sffamily\color{textcolor}
  },
  commit/.style={
    rectangle,
    draw=commitdraw,
    fill=commitbg,
    rounded corners=6pt,
    line width=1.2pt,
    text width=4.6cm,
    align=left,
    inner sep=10pt,
    font=\sffamily\color{textcolor}
  },
  quarantine/.style={
    rectangle,
    draw=quarantinedraw,
    fill=quarantinebg,
    rounded corners=6pt,
    line width=1.2pt,
    text width=4.8cm,
    align=left,
    inner sep=10pt,
    font=\sffamily\color{textcolor}
  },
  arrow/.style={
    ->,
    >=Stealth,
    line width=1.2pt,
    draw=textcolor!80
  },
  line/.style={
    line width=1.2pt,
    draw=textcolor!80
  },
  labelnode/.style={
    font=\sffamily\scriptsize,
    text=textcolor,
    align=center,
    fill=white,
    inner sep=3pt
  }
}

\hypersetup{colorlinks=true,linkcolor=blue,citecolor=blue,urlcolor=blue}
\setlist[itemize]{leftmargin=1.4em}
\setlist[enumerate]{leftmargin=1.6em}

\newtheorem{definition}{Definition}[section]
\newtheorem{theorem}{Theorem}[section]
\newtheorem{proposition}{Proposition}[section]

\newcommand{\sem}[1]{[\![#1]\!]}

\lstdefinelanguage{Julia}{
  morekeywords={abstract,break,case,catch,const,continue,do,else,elseif,end,export,false,for,function,global,if,import,importall,in,macro,module,local,mutable,struct,primitive,quote,return,true,try,type,typealias,using,while},
  sensitive=true,
  morecomment=[l]{\#},
  morecomment=[s]{\#=}{=\#},
  morestring=[b]",
  morestring=[m]'
}

\lstdefinestyle{juliastyle}{
  basicstyle=\ttfamily\scriptsize,
  breaklines=true,
  columns=fullflexible,
  frame=single,
  showstringspaces=false,
  keywordstyle=\color{blue!60!black},
  commentstyle=\color{green!40!black},
  stringstyle=\color{purple!60!black}
}
\title{Logit-Boundary Geometric Belief Interfaces and Sparse Sheaf-Enclave Protocols:\\A Self-Contained Substrate for Secure Network Electronic Health Record (EHR) Interoperability}
\author{Alvin Spivey \& Yu Huang\\Light Imaging Technologies, Inc.\\\texttt{alvin@lightimagingtech.com}}
\date{June 30, 2026\\Revised August 10, 2026}

\begin{document}
\maketitle

\begin{center}
\begin{minipage}{0.92\textwidth}
\small\centering
\PublicationLicenseNotice\\[3pt]
\textit{Research manuscript. Not a clinical decision system, medical device, or authorization for autonomous EHR write-back.}
\end{minipage}
\end{center}
\medskip

\begin{abstract}
Electronic health-record interoperability is a boundary problem: legacy systems, generative models, terminology services, identity systems, and human reviewers may each expose rich internal states, while operational exchange requires a narrow shared interface of typed claims, bounded uncertainty, provenance, and explicit admission or abstention. This paper develops a mathematical and engineering architecture for that interface. The organizing idea is the \emph{logit boundary}: a discovery model may propose pre-threshold scores over a local categorical decision, but a deterministic judgment substrate decides whether the proposal is admissible, requires review, or must be quarantined before any Fast Healthcare Interoperability Resources (FHIR) transaction is constructed. The resulting Geometric Belief Interface (GBI) combines finite boundary semantics, local Dirichlet evidence, cellular-sheaf and mapping-cone diagnostics, advisory geometric audit charts, and a Decentralized Cryptographic Sheaf-Enclave (DCSE) protocol sketch for fail-closed deployment. The framework does not establish clinical truth, global representation alignment, or end-to-end clinical safety; it defines certificate-producing checks at a model-to-system boundary. A companion frozen synthetic benchmark, GBI BoundaryBench v0.1, evaluated Qwen3-4B-Instruct-2507 on 256 held-out tasks across three evidence modes (768 canonical executions). All executions completed, but none produced an output accepted by the benchmark contract: 369 were rejected during safe parsing and 399 during schema validation, yielding zero coverage and deterministic quarantine. This empirical result is deliberately narrow---one 4B open-weight model under one frozen interface---and is reported as evidence about the admission boundary, not as a general claim about LLM capability or clinical safety. A Julia appendix verifies selected numerical certificates using standard libraries.  For codebase, reference \url{https://github.com/AlvinSpivey/GBI-BoundaryBench}.
\end{abstract}

\tableofcontents

\section{The organizing picture: logits as the clinical boundary}

A neural model does not usually expose its full internal state. Open-weight inference can expose full pre-threshold logits, while hosted APIs may expose only top-$k$ log probabilities or the generated output. We use the \emph{logit boundary} as the conceptual output boundary and require the receipt to record exactly which evidence was actually available; unavailable logits are never reconstructed or treated as observed. When full logits are available, they are pre-threshold scores assigned to a finite set of possible outputs. For a clinical parser, those outputs might be
\[
\{\texttt{exact},\texttt{equivalent},\texttt{narrower},\texttt{broader},\texttt{conflict},\texttt{unmapped}\}.
\]
For an allergy classifier, they might be
\[
\{\texttt{confirmed-active},\texttt{unconfirmed},\texttt{refuted},\texttt{historical-resolved}\}.
\]
The logit vector is not a fact. It is a numerical proposal over a local categorical boundary. The architecture in this paper treats it exactly that way.

The practical division of labor is:
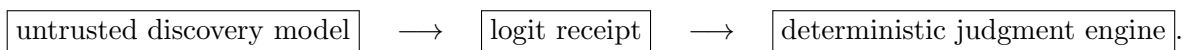
\[
\boxed{\text{untrusted discovery model}}\quad\longrightarrow\quad
\boxed{\text{logit receipt}}\quad\longrightarrow\quad
\boxed{\text{deterministic judgment engine}}.
\]
The discovery model may be an LLM, a rules parser, an embedding model, or a specialist clinical classifier. It proposes local categorical scores. The judgment engine checks identity, terminology version, provenance, temporal scope, sheaf consistency, and policy. Only then may the system update evidence or construct a FHIR transaction.

This logit-boundary language makes the architecture easier to communicate. It gives neural, symbolic, and distributed components a common interface:
\[
\text{finite categories} + \text{scores} + \text{calibration} + \text{provenance}.
\]
The topology and geometry enter after this boundary is fixed. They are not asked to prove clinical truth by themselves.

\subsection{Reader map}

The mathematical stack is as follows:
\[
\begin{array}{ll}
\text{logit topology} & \text{how neural outputs become finite boundary evidence},\\
\text{finite semantics} & \text{how local clinical propositions are typed},\\
\text{Dirichlet evidence} & \text{how uncertainty is calibrated over local categories},\\
\text{sheaf diagnostics} & \text{how local claims fail to glue globally},\\
\text{geometric charts} & \text{how high-risk contradictions are visualized, not decided},\\
\text{DCSE protocols} & \text{how identity and audit certificates are replicated safely},\\
\text{FHIR application} & \text{how the whole system behaves in EHR interoperability}.
\end{array}
\]

\section{Logit topology for bounded clinical discovery}

The Universal Representation Hypothesis motivates testing whether diverse neural architectures trained on identical or overlapping data distributions exhibit
comparable task-facing structure; the framework here does not assume global manifold congruence. Even when models share task and data constraints, their internal geometry is not
identifiable from output agreement alone.

At the logit boundary, downstream behavioral equivalence is treated as an observable interface condition. Identical correct predictions establish agreement only of
selected outputs; exact pre-threshold-logit agreement is the stronger condition defined below. Neither condition alone implies homeomorphic or diffeomorphic alignment
of hidden representations, preservation of distances or neighborhoods, or a shared semantic hierarchy. The bounded claim used here is therefore only that models can be
compared through the output quantities they expose on a specified stimulus set.

\subsection{Exact logit equivalence and probe-visible quotients}

Let models $A$ and $B$ have hidden spaces $\mathcal H_A=\mathbb R^{D_A}$ and $\mathcal H_B=\mathbb R^{D_B}$ with affine readouts $(W_A,b_A)$ and $(W_B,b_B)$. The comparison below is intentionally restricted to an evaluated stimulus set $\mathcal X$.

\begin{definition}[Exact logit equivalence]
Models $A$ and $B$ are exactly logit-equivalent on $\mathcal X$ if
\[
W_Ah_A(x)+b_A=W_Bh_B(x)+b_B,
\qquad \forall x\in\mathcal X.
\]
This is stronger than agreement of argmax predictions or softmax probabilities.
\end{definition}

\begin{theorem}[Affine reconstruction on the evaluated set]
If the two models are exactly logit-equivalent on $\mathcal X$ and $W_B\in\mathbb R^{M\times D_B}$ has full column rank, then
\[
h_B(x)=W_B^+W_Ah_A(x)+W_B^+(b_A-b_B),
\qquad x\in\mathcal X,
\]
where $W_B^+=(W_B^\top W_B)^{-1}W_B^\top$.
\end{theorem}

\begin{proof}
Exact logit equivalence gives
\[
W_Bh_B(x)=W_Ah_A(x)+b_A-b_B.
\]
Multiplication by the left inverse $W_B^+$ yields the stated reconstruction because $W_B^+W_B=I_{D_B}$.
\end{proof}

The theorem is a statement about the observed stimulus set. It does not imply that the induced linear map is invertible, that it extends uniquely beyond $\mathcal X$, or that hidden-state distances, neighborhoods, topology, or semantic organization are preserved.

\begin{definition}[Probe-visible quotient]
For a linear probe family $V\subseteq\mathcal H^*$, define
\[
K(V)=\{h\in\mathcal H:\ell(h)=0\text{ for every }\ell\in V\},
\qquad
Z(V)=\mathcal H/K(V).
\]
For a readout matrix $W$, taking $V$ to be the span of its rows removes directions that are invisible to that readout.
\end{definition}

Thus, output agreement supplies a bounded observable comparison object. It does not by itself establish clinical correctness or global cross-model representation alignment.

\subsection{Affine readout and the pre-threshold boundary}\label{sec:affine-readout}

In modern deep learning architectures, the terminal layer functions as a map that projects latent representations onto a categorical space. Let $h(x) \in \mathcal{H} \subset \mathbb{R}^D$ define the hidden representation vector of an input stimulus $x \in \mathcal{X}$ within a $D$-dimensional Hilbert space $\mathcal{H}$, commonly termed the residual stream or embedding manifold. The transformation mapping this hidden state to the pre-threshold logit space $\mathcal{L} \subset \mathbb{R}^M$, where $M = |\mathcal{V}|$ represents the cardinality of the model's vocabulary or target category set, is defined by the affine readout map:
\[
L(h(x)) = W_U \cdot h(x) + b
\]
where $W_U \in \mathbb{R}^{M \times D}$ denotes the linear unembedding weight matrix and $b \in \mathbb{R}^M$ represents the bias vector. Stated again, for simplicity, we're letting
\[
  h(x)\in\mathcal H\cong \mathbb R^D
\]
be a hidden representation of a stimulus, such as a note fragment, medication string, or context window. A model's readout head maps this hidden state to logits
\[
  L(h(x)) = W h(x)+b \in \mathbb R^M,
\]
where \(W\in\mathbb R^{M\times D}\), \(b\in\mathbb R^M\), and \(M\) is the number of local categories. The associated probability vector then becomes
\[
  p_i = \operatorname{softmax}(L)_i
      = \frac{\exp(L_i/\tau)}{\sum_{j=1}^M \exp(L_j/\tau)},
\]
with temperature \(\tau>0\). When full logits are available, retaining both logits and the associated probability vector preserves information that hard thresholding would discard. When they are unavailable, the receipt records the lower-information evidence mode explicitly rather than fabricating logits.

\begin{definition}[Logit receipt]
A \emph{logit receipt} for a local decision consists of
\[
  R=(C,L,p,\tau,k,m,\rho),
\]
where \(C=\{c_1,\ldots,c_M\}\) is the local category set, \(L\in\mathbb R^M\) are logits, \(p=\operatorname{softmax}(L/\tau)\), \(k\) is the top-\(k\) truncation level if any, \(m\) is model metadata, and \(\rho\) is cryptographic provenance binding the input, terminology bundle, prompt, model version, and runtime policy.
\end{definition}

A receipt is evidence, not authority. For providers that do not expose full logits, the envelope records the available output-only or top-$k$ evidence and marks unavailable fields as unavailable. The judgment engine may ignore a receipt, quarantine it, ask a human to review it, or convert supported evidence into a Dirichlet update (see Section~\ref{sec:Dirichlet}).

\subsection{Cross-model interpretation}

Exact logit agreement is a property of an evaluated output boundary, not a proof that two models implement the same internal representation or that either model is clinically correct. When only probabilities are observed, the common-logit-shift symmetry of softmax must also be accounted for. When a readout is rank deficient, the natural comparison is the readout image, or equivalently the quotient by readout-invisible directions. Transporting a behavioral or safety probe between architectures requires an explicitly identified inter-model correspondence and validation on the target model; the quotient construction alone does not imply zero-label probe portability.

\subsection{A numerical example: affine equivalence}

Take \(D=3\), \(M=6\). Let \(W_B\) be full column rank, choose an affine transform \(T,v\), and set
\[
  W_A=W_BT,
  \qquad
  b_A=W_Bv+b_B.
\]
Then
\[
  W_Ah_A+b_A=W_B(Th_A+v)+b_B.
\]
A numerical verification of this construction gives
\[
\|L_A-L_B\|_2\approx 1.49\times 10^{-15},
\]
which is consistent with floating-point roundoff in the worked example. This is ordinary linear algebra, but it is conceptually important: exact logit agreement supplies a shared observable boundary for local categorical behavior on the evaluated stimulus set; it does not establish hidden-state isomorphism beyond the assumptions above.

\subsection{Softmax dynamics and top-\texorpdfstring{\(k\)}{k} loss}

Logits are smooth; hard decisions are not. As \(\tau\to0\),
\[
  \operatorname{softmax}(L/\tau)\longrightarrow e_{\arg\max_i L_i},
\]
which collapses a continuous score vector to a vertex of the simplex. For clinical review this matters. Top-\(k\) display can hide rare but safety-critical alternatives.

For the example
\[
  L=(4.0,2.7,1.4,0.7,-0.2,-1.0),
\]
Using natural logarithms, the entropy values below are in nats:
\[
\begin{array}{c|c|c}
\tau & H(\operatorname{softmax}(L/\tau)) & \max_i p_i\\\hline
1.00 & 0.885219 & 0.711530\\
0.50 & 0.293843 & 0.924709\\
0.20 & 0.011293 & 0.998496\\
0.05 & 0.000000 & 1.000000.
\end{array}
\]
The top-3 tail mass at \(\tau=1\) is about \(0.0417\). A hard top-3 truncation sets nonzero probabilities to zero, making \(D_{\rm KL}(P\|P^{(3)})\) infinite. Therefore a clinical logit receipt must include full-spectrum logits when feasible, or at least tail mass, calibration, and a clear statement of discarded alternatives.

\subsection{Component logits: prisms of model interpretation}

If a model exposes component activations, the logit boundary can be decomposed. Suppose an approximate residual expansion has the form
\[
  h^{(L)} = h^{(0)} + \sum_{\ell=1}^{L} a^{(\ell)} + \sum_{\ell=1}^{L} m^{(\ell)},
\]
where \(a^{(\ell)}\) is an attention-block contribution and \(m^{(\ell)}\) is an MLP contribution. Because the readout is affine,
\[
  L = Wh^{(0)} + \sum_{\ell=1}^{L} W a^{(\ell)} + \sum_{\ell=1}^{L} W m^{(\ell)} + b.
\]
This gives a practical audit primitive \cite{logitlens}. A clinical model can report not only that \lstinline[language=Julia] |"the proposed category is in conflict"|, but also which model components pushed the logit toward that category. The GBI does not need to believe these explanations. It can store them as extra receipt fields and subject them to the same deterministic review.

\subsection{Logit space as a non-autonomous dynamical system}

Autoregressive and interactive models are input-driven dynamical systems. In a simplified discrete form,
\[
  h_{t+1}=\Phi_t(h_t,x_t),
  \qquad
  L_t = Wh_t+b,
  \qquad
  p_t=\operatorname{softmax}(L_t/\tau).
\]
For intuition, one may define category potentials
\[
  U_c(h)=-L_c(h).
\]
A high logit is a low potential well. Sequential inputs move the wells, so the trajectory may drift between basins. A two-dimensional illustrative system can exhibit the same qualitative effect---for example, a category switch while entropy remains high---but the present appendix does not implement that dynamical example. This is exactly the setting where deterministic clinical judgment is needed; a smooth logit trajectory can still cross a discrete decision boundary.

\subsection{What the universal logit topology contributes}

The logit-topology framing contributes four concrete design rules. First, compare AI systems at their output boundary, not by assuming their hidden states have the same coordinates. Second, treat hidden-state alignment as quotient alignment through probe-visible subspaces. Third, preserve full-spectrum or tail-mass evidence because top-\(k\) displays may hide safety-relevant alternatives. Fourth, treat neural outputs as dynamical proposals that must be judged by finite semantics, sheaf diagnostics, identity certificates, and FHIR policy.

\section{Finite boundary semantics}

From Section~\ref{sec:affine-readout}, we defined an individual logit affine readout map $L(h(x)) = W_U \cdot h(x) + b$, where $W_U \in \mathbb{R}^{M \times D}$ denoted the linear unembedding weight matrix and $b \in \mathbb{R}^M$ represented the bias vector. 

In the regime $D\ll M$, $\operatorname{rank}(W_U)\leq D$. The affine image is $L(\mathcal H)=b+\operatorname{Im}(W_U)$, whose dimension is $\operatorname{rank}(W_U)$ rather than necessarily $D$.

Mathematically, this projection partitions the latent embedding space $\mathcal{H}$ into two orthogonal vector subspaces: the active row space of $W_U$ and the null space or kernel, defined as:
\[
\ker(W_U) = \{ v \in \mathbb{R}^D \mid W_U v = 0 \}
\]
Any variation in the latent hidden representation $h(x)$ that lies within $\ker(W_U)$ is annihilated during the transformation, exerting zero influence on the output logit vector. This kernel constitutes a functionally latent, ``probe-invisible'' subspace. 

The probe-visible quotient space relates to the logits spectrum of all available neural networks through the foundational geometric principle of \textbf{readout equivalence} and coordinate-free representation mapping.

For any neural network, the terminal hidden state $h(x) \in H$ is projected onto the $M$-dimensional logit space $\mathcal{L}$ via an affine readout head, defined as:
\[
L(h(x)) = W_U \cdot h(x) + b
\]
Because neural networks differ arbitrarily in their initialization, parameterization, and layer-wise architectures, their internal hidden coordinates are not directly comparable. However, the logit spectrum serves as a shared, behaviorally observable interface through which different networks can be compared.

If we define a probe family $V$ as the linear span of the readout directions (the rows of the unembedding matrix $W_U$), the hidden dimensions that do not affect the output logits constitute the probe-invisible subspace, $K(V) = \ker(W_U)$. Factoring out this null space leaves the \textit{probe-visible quotient space}:
\[
Z(V) = H / K(V) = H / \ker(W_U)
\]
This quotient space relates to the logit spectrum of all available neural networks through three key mathematical and operational principles:

\begin{enumerate}
    \item The Canonical Linear Isomorphism
    
    By the first isomorphism theorem, the readout linear map $W_i:H_i\to\mathbb R^M$ induces a canonical linear isomorphism
    \[
    \bar W_i:H_i/\ker(W_i)\longrightarrow\operatorname{im}(W_i),\qquad [h]\longmapsto W_i h.
    \]
    Thus the probe-visible quotient is canonically identified with the image of the readout linear map; no undefined dual space is required.
    
    \item Affine Equivalence Under Output Agreement
    
    Under exact output agreement, the earlier affine reconstruction is justified when the relevant readout has full column rank. Without that assumption, quotienting identifies readout-invisible directions, but output agreement alone does not establish an invertible hidden-state transformation between models. Softmax also has a common-logit-shift symmetry that should be treated explicitly when probabilities rather than logits are compared.
    
    \item Cross-Model Probe Portability and Transfer
    
    Because the full hidden state $H$ contains model-specific probe-invisible directions, transferring a probe between hidden spaces may fail unless an inter-model correspondence is identified and validated.
\end{enumerate}

A shared probe-visible quotient can serve as a candidate comparison object, but transporting a safety or behavioral monitor between architectures requires an explicitly identified inter-model map and target-model validation. Zero-label portability is not implied by the quotient construction alone.

\subsection{Data Semantic Representation}

This means that a clinical interoperability substrate cannot begin with vague natural language. It needs a small typed boundary.

\begin{definition}[Boundary algebra]
A boundary algebra is a finite Boolean algebra \(B\). Its atoms are
\[
  \operatorname{At}(B)=\{a_1,\ldots,a_N\},
\]
and every element of \(B\) is a union of atoms.
\end{definition}

\begin{definition}[Model-relative semantics]
Let \(W\) be a set of possible clinical worlds, \(T\) a time domain, and \(L\) a location or facility domain. A semantics for \(B\) is a Boolean homomorphism
\[
  \sem{\cdot}_{\mathcal M}:B\to \mathcal P(W\times T\times L),
\]
so that
\[
\begin{aligned}
\sem{b\wedge c}&=\sem{b}\cap\sem{c},\\
\sem{b\vee c}&=\sem{b}\cup\sem{c},\\
\sem{\neg b}&=(W\times T\times L)\setminus\sem{b}.
\end{aligned}
\]
\end{definition}

This semantics is model-relative. It does not prove reality. It says exactly what the system means when it states, for example, that an allergy is active during a half-open time interval \([t_0,t_1)\) at a facility.

\section{Condensed-mathematics motivation and operational probes}

Classical point-set topology treats a space as a set of points with open subsets. Condensed mathematics changes the perspective: a space is tested by maps from compact totally disconnected spaces \cite{scholze_condensed,scholze_lectures}. For a compactly generated Hausdorff space \(X\), one writes
\[
  \underline X(S)=C^0(S,X)
\]
for a profinite probe \(S\).

The intuition is simple. A single point may not reveal the topology of a system. A convergent sequence can. Let \(\mathbb R_\delta\) be the real numbers with the discrete topology and let
\[
  f:\mathbb R_\delta\to\mathbb R
\]
be the identity map of sets. Algebraically, it looks bijective. Topologically, it is not an isomorphism. Use the profinite convergent sequence
\[
  \mathbb N_\infty=\mathbb N\cup\{\infty\}.
\]
For the cokernel statement below, regard $\mathbb R_\delta$ and $\mathbb R$ as topological abelian groups and form the associated condensed abelian groups; the cokernel is taken in that additive category, not in plain set-valued condensed spaces. Then
\[
\operatorname{Cok}(\underline f)(\mathbb N_\infty)
\cong
\frac{\{\text{convergent real sequences}\}}
     {\{\text{eventually constant real sequences}\}},
\]
which is nonzero. The probe detects the mismatch.

In software, we cannot execute infinite profinite objects. We implement \emph{operational profinite probes}: finite, clock-bounded audit loops that approximate the same testing philosophy.

\begin{definition}[Operational profinite probe]
An operational profinite probe is a finite sequence
\[
  P=(s_0\to s_1\to\cdots\to s_n)
\]
of version-pinned tests, such as terminology snapshots, identity claims, policy versions, or FHIR capability states. A probe passes if every transition preserves the declared boundary invariants.
\end{definition}

This is an engineering analogy motivated by condensed-mathematics probing: the system is tested by structured families of contexts rather than isolated examples. The finite operational probe defined here is not claimed to approximate a condensed object in a formal convergence sense, and no condensed-mathematics theorem is used as a clinical safety guarantee.

\section{Neuro-symbolic partitioning: discovery and judgment}

The architecture separates two roles.

\paragraph{Ars inveniendi: discovery.}
An LLM, parser, embedding model, or classifier may read free text and emit logit receipts. It is an untrusted compiler from messy input to typed proposals. Its output must include model version, prompt digest, terminology version, full or bounded-tail logits, and provenance.

\paragraph{Ars iudicandi: judgment.}
A deterministic engine checks whether the receipt is admissible. It performs type checking, terminology validation, Dirichlet evidence update, sheaf consistency diagnostics, identity-log verification, and FHIR transaction construction. The judgment engine is the only component allowed to create commit-ready clinical payloads.

The boundary between the two layers is a signed JSON-LD envelope:
\[
\mathcal E=(\text{subject},\text{interval},\text{facility},\text{category set},L,p,\rho,\text{policy}).
\]
This envelope can be reviewed by humans, tested by finite probes, and attached to FHIR \texttt{Provenance} or \texttt{AuditEvent} records \cite{fhir_r4,fhir_provenance}.

\section{Local hierarchical Dirichlet evidence}\label{sec:Dirichlet}

A Dirichlet-multinomial (also called a Dirichlet compound multinomial) is a probability distribution used to model categorical count data that exhibits overdispersion (more variance than a standard model would expect).

To understand the Dirichlet-multinomial, think of it as a two-step hierarchical process:
\begin{enumerate}
	\item The Dirichlet Step: Instead of assuming all observations come from one fixed set of underlying probabilities, we assume that the underlying probability vector (the chances of a specific event occurring) is continuously changing or uncertain. We draw this probability vector from a Dirichlet distribution.
	\item The Multinomial Step: Once we have our specific probability vector, we use it to generate our discrete counts using a Multinomial distribution.
\end{enumerate}

A standard multinomial distribution assumes all data is generated from one fixed, known set of probabilities. Because of this, it can often underestimate the true variability of real-world data (a problem called \textit{overdispersion}).

By injecting uncertainty about the probabilities themselves (using the Dirichlet step), the Dirichlet-multinomial allows for much wider variance. It accounts for instances where groups or subjects vary much more widely in their categorical counts than a simple multinomial would predict. At the logit boundary of statistical artificial intelligence methods, this wider variance persists.

\subsection{Local categorical decisions}

A Dirichlet distribution belongs on a local categorical simplex. Let
\[
  C=\{c_1,\ldots,c_K\}
\]
be mutually exclusive and exhaustive alternatives for one decision. Then
\[
  p\in\Delta^{K-1}=\left\{p\in\mathbb R_{\ge0}^K:\sum_{i=1}^K p_i=1\right\},
\]
and
\[
  p\sim\operatorname{Dir}(\alpha),
  \qquad
  \alpha\in\mathbb R_{>0}^K.
\]
The Fisher information matrix in \(\alpha\)-coordinates is
\[
  I(\alpha)_{ij}=\psi_1(\alpha_i)\delta_{ij}-\psi_1(\alpha_0),
  \qquad
  \alpha_0=\sum_i\alpha_i,
\]
where \(\psi_1\) is the trigamma function.

\subsection{Evidence boxes and dynamic atom registries}

A fixed evidence box
\[
  \alpha_i\in[\varepsilon,A]
\]
prevents boundary singularities. However, clinical terminologies split and merge. A rigid category set can create schema lock-in. We therefore use a hierarchical Dirichlet-multinomial registry.

At time \(t\), let \(C_t\) be the active category set. When a new category \(c_{K+1}\) appears, the registry assigns
\[
  \alpha_{K+1}=\alpha_{\rm new}>0
\]
and records the terminology version and parent category. This avoids division by zero and makes category growth auditable.

\subsection{Numerical conditioning}

The verification script compares
\[
  \alpha=(2,3,4,5)
\]
with
\[
  \alpha=(0.01,3,4,5).
\]
The first has condition number about \(20.46\). The second has condition number about \(4.55\times10^5\). This is the practical reason for an evidence box: near-boundary exclusion can make Fisher geometry numerically unstable.

\subsubsection{Expanded: Numerical Conditioning in the Dirichlet Fisher Metric}

For a \(K\)-category Dirichlet distribution with parameter
\[
\alpha=(\alpha_1,\ldots,\alpha_K)\in\mathbb R_{>0}^K,
\qquad
\alpha_0=\sum_{i=1}^K\alpha_i,
\]
the Fisher information matrix in \(\alpha\)-coordinates is
\[
g_{ij}(\alpha)
=
\psi_1(\alpha_i)\delta_{ij}
-
\psi_1(\alpha_0),
\]
where \(\psi_1\) is the trigamma function. Equivalently,
\[
g(\alpha)
=
\operatorname{diag}\big(\psi_1(\alpha_1),\ldots,\psi_1(\alpha_K)\big)
-
\psi_1(\alpha_0)\mathbf 1\mathbf 1^\top.
\]

The condition number is
\[
\kappa_2(g)
=
\frac{\lambda_{\max}(g)}{\lambda_{\min}(g)}.
\]

The verification script compares two parameter vectors:
\[
\alpha^{(a)}=(2,3,4,5),
\]
and
\[
\alpha^{(b)}=(0.01,3,4,5).
\]

For the moderate interior point \(\alpha^{(a)}=(2,3,4,5)\), one obtains approximately
\[
\psi_1(\alpha^{(a)})
=
(0.644934,\;0.394934,\;0.283823,\;0.221323),
\]
and
\[
\psi_1(\alpha_0)=\psi_1(14)\approx 0.074040.
\]
The eigenvalues of \(g(\alpha^{(a)})\) are approximately
\[
(0.029494,\;0.254435,\;0.361568,\;0.603356).
\]
Therefore
\[
\kappa_2(g(\alpha^{(a)}))
=
\frac{0.603356}{0.029494}
\approx 20.46.
\]

For the near-boundary point
\[
\alpha^{(b)}=(0.01,3,4,5),
\]
the first trigamma value is enormous:
\[
\psi_1(0.01)\approx 10001.621.
\]
This follows from the asymptotic behavior
\[
\psi_1(x)\sim \frac{1}{x^2}
\qquad
\text{as }x\downarrow 0.
\]
Thus
\[
\psi_1(0.01)\approx \frac{1}{(0.01)^2}=10^4,
\]
up to lower-order terms.

The resulting eigenvalues are approximately
\[
(0.021986,\;0.254705,\;0.362908,\;10001.5344).
\]
Hence
\[
\kappa_2(g(\alpha^{(b)}))
=
\frac{10001.5344}{0.021986}
\approx
4.55\times 10^5.
\]

If an implementation adopts a condition-number budget, the lower evidence-box bound $\varepsilon$ can be selected by solving for the smallest value that keeps $\kappa_2(g)$ below that budget over the declared parameter region. In the one-dimensional sweep $[\varepsilon,3,4,5]$, a budget of $10^4$ is crossed near $\varepsilon\approx0.066$. This number is illustrative rather than a universal certification threshold. The Appendix~\ref{app:code} deliberately evaluates $\alpha_1=0.01$ as a near-boundary stress case; the reference script does not itself enforce an operational evidence box.

Additionally, this is the practical reason for an evidence box. As any \(\alpha_i\) approaches zero, the Dirichlet Fisher geometry becomes extremely anisotropic. A small movement in the nearly excluded category direction has enormous Fisher cost, while other directions remain at ordinary scale. Numerical optimization then becomes stiff: gradient steps, Newton solves, natural gradient updates, and uncertainty propagation can become dominated by the singular boundary coordinate.

It is more accurate to call this a \emph{near-boundary sparsity} or \emph{near-boundary exclusion} effect, rather than ordinary certainty. In a Dirichlet model, \(\alpha_i\ll 1\) places strong mass near the simplex face \(p_i=0\). The category is nearly excluded, and the Fisher metric becomes singular at that face.

The evidence box
\[
\alpha_i\in[\varepsilon,A]
\]
keeps every parameter away from zero and infinity. Since \(\psi_1\) is continuous and positive on compact subsets of \((0,\infty)\), the Fisher matrix then has finite eigenvalue bounds:
\[
0<\lambda_*(K,\varepsilon,A)
\le
\lambda_{\min}(g(\alpha))
\le
\lambda_{\max}(g(\alpha))
\le
\lambda^*(K,\varepsilon,A)
<\infty.
\]
Therefore,
\[
\kappa_2(g(\alpha))
\le
\frac{\lambda^*(K,\varepsilon,A)}{\lambda_*(K,\varepsilon,A)}
<\infty.
\]

The evidence box is not merely a numerical trick. It is a certificate that every categorical evidence state remains inside a computationally stable region of the Fisher--Dirichlet manifold.

\section{Finite sheaves and mapping-cone diagnostics}

Finite sheaves provide a structured way to represent local data constraints and their compatibility. In the intended EHR setting, a contradiction should be localized before a governed write is constructed, so that independent work items can be routed separately when policy permits. These diagnostics do not override the atomic semantics of a FHIR transaction Bundle; any stalk-level quarantine must occur before transaction construction or across separately scoped transactions or batch work.

\subsection{Cellular sheaves in one page}

A finite cellular sheaf assigns a vector space to each cell and a linear restriction map to each incidence \cite{hansen_ghrist_spectral}. In this paper, every stalk is a finite-dimensional real inner-product vector space. This is essential because Laplacians require adjoints.

Let \(X\) be a finite cell complex and \(\mathcal F\) a sheaf on \(X\). The cochain space is
\[
  C^k(X;\mathcal F)=\bigoplus_{\sigma\in X_k}\mathcal F(\sigma).
\]
The coboundary \(\delta^k:C^k\to C^{k+1}\) is assembled from restriction maps. The Hodge Laplacian is
\[
  \Delta^k=(\delta^{k-1})(\delta^{k-1})^*+(\delta^k)^*\delta^k.
\]
Its kernel represents cohomology \cite{hansen_ghrist_spectral}:
\[
  \ker\Delta^k\cong H^k(X;\mathcal F).
\]

\subsection{Mapping cones}

Let \(\varphi:\mathcal F\to\mathcal G\) be a sheaf morphism. The mapping cone measures how far \(\varphi\) is from gluing consistently \cite{weibel_homological}. Its cochains are
\[
  C^q(\operatorname{Cone}\varphi)=C^{q+1}(X;\mathcal F)\oplus C^q(X;\mathcal G),
\]
with differential
\[
  d_{\rm cone}(x,y)=(-d_{\mathcal F}x,\varphi x+d_{\mathcal G}y).
\]
The cone Laplacian has a null space of relative obstructions.

\begin{definition}[Trace cell energy]
Let \(\Pi_\lambda\) be the orthogonal projector onto the obstruction null space of a cone Laplacian, and let \(\Pi_\sigma\) be the projector onto a stalk subspace. Define
\[
  E_\sigma=\operatorname{tr}(\Pi_\lambda\Pi_\sigma).
\]
\end{definition}

\begin{proposition}[Basis invariance]
The quantity \(E_\sigma\) is invariant under orthogonal rotation of any basis chosen for the obstruction subspace.
\end{proposition}

\begin{proof}
Only the projector \(\Pi_\lambda=UU^\top\) matters. Replacing \(U\) by \(UQ\) for \(Q\in O(m)\) gives
\[
  UQQ^\top U^\top=UU^\top.
\]
Therefore \(\operatorname{tr}(\Pi_\lambda\Pi_\sigma)\) is unchanged.
\end{proof}

\subsection{Graceful degradation}

For a FHIR Bundle submitted with \texttt{type=transaction}, atomic rollback remains the transaction semantic \cite{fhir_r4}. Stalk-level quarantine is applied before constructing that atomic transaction, or across independently scoped transactions or batch work. If
\[
  E_\sigma>\theta_{\rm quarantine},
\]
only the affected stalk is quarantined. Other independent stalks may proceed if policy allows.

The executable appendix includes a \emph{projector-based toy obstruction surrogate} with three coordinate stalks,
\[
  \text{Allergy},\quad \text{MedicationRequest},\quad \text{RenalLab}.
\]
After orthonormalizing the printed obstruction basis, the trace energies are
\[
\begin{array}{c|c|c}
\text{stalk} & E_\sigma & \text{toy gate}\\\hline
\text{Allergy} & 0.019778 & \text{eligible}\\
\text{MedicationRequest} & 1.977750 & \text{quarantine}\\
\text{RenalLab} & 0.002472 & \text{eligible}.
\end{array}
\]
These values illustrate basis-invariant localization, but the script does not construct a sheaf morphism, a mapping-cone differential, or a mapping-cone Hodge Laplacian. Therefore this numerical table is not evidence that a deployed mapping-cone diagnostic has been validated. ``Eligible'' means only that an independent work item could proceed to later policy gates; it does not mean partial success inside a single atomic FHIR transaction.

\section{Geometry for audit and state charts}

The EHR label adjudicator's interface shows a tabular contradiction report: stalk energies, offending restrictions, terminology versions, and provenance. The clinical decision is made from discrete facts and policy. These reports are generated from version-pinned mappings, restriction checks, and provenance available to the adjudication workflow.

Using Higher-dimensional hyperellipsoids, identity certificates and audit certificates are generated digitally and used to support the authenticity, accountability, and security of EHR systems. They act as the digital passport and tamper-evident logbook for clinical data. Identity certificates are cryptographic credentials assigned to a specific patient.

Audit certificates are tamper-evident records or cryptographically signed metadata logs used to record and examine system activity involving patient information. An implementation may record user identity, timestamp, affected resource or action, source/device context, and provenance according to institutional policy and applicable requirements. HIPAA Security Rule audit controls require mechanisms to record and examine activity in systems that use or contain electronic protected health information \cite{hipaa_audit_controls}; this manuscript does not claim that HIPAA universally mandates the exact field set or timestamp granularity stated in the superseded text. Such logs support security, compliance, and forensic review.

\subsection{Higher-dimensional hyperellipsoid certificates}

For an \(n\)-dimensional state chart, the correct quasiconformal object is an orientation-preserving homeomorphism \cite{reshetnyak_distortion,vaisala}
\[
  f:\Omega\subset\mathbb R^n\to\Omega'\subset\mathbb R^n
\]
with \(f\in W^{1,n}_{\rm loc}\) and
\[
  \|Df(x)\|^n\le KJ_f(x)\quad\text{a.e.}
\]
At differentiability points, \(Df(x)\) maps infinitesimal spheres to ellipsoids. If the singular values are
\[
  \sigma_1\ge\cdots\ge\sigma_n>0,
\]
then a practical certificate is
\[
  H_f(x)=\frac{\sigma_1}{\sigma_n}.
\]
A safety-critical implementation must also check positive Jacobian, boundary behavior, inverse residuals, and stratum preservation.

For the matrix
\[
A=\begin{bmatrix}
1.20&0.10&0\\
0.20&0.80&0.05\\
0&0.10&1.10
\end{bmatrix},
\]
the script computes singular values approximately
\[
  (1.254966,1.110695,0.737507),
\]
so
\[
  H\approx1.701632,
  \qquad
  J\approx1.028000,
  \qquad
  K_O=\frac{\sigma_1^3}{J}\approx1.922661.
\]

\subsection{Geometry for Audit and State Charts}

The geometry of these certificates can be useful or validation and interpretability.  The EHR label adjudicator can display a tabular contradiction report containing:
\[
\left(
\begin{array}{l}
\text{stalk energies }E_\sigma,\\
\text{offending restriction maps},\\
\text{terminology versions},\\
\text{source and target code systems},\\
\text{FHIR resource identifiers},\\
\text{Provenance references},\\
\text{policy rule identifiers},\\
\text{recommended abstention or review action}
\end{array}
\right).
\]

The clinical decision is made from discrete facts, policy, and human authority. The chart may orient the reviewer, but it should not authorize a write.

Reports are generated deterministically by comparing candidate local sections against the version-pinned reference EHR and terminology sheaves. Formally, let
\[
\epsilon:\mathcal F_{\mathrm{cand}}\longrightarrow \mathcal W_{\mathrm{ref}}
\]
be the grounding morphism from candidate clinical assertions to the authoritative reference sheaf. The mapping-cone Laplacian \(L_C\) exposes relative inconsistencies. For each clinical stalk \(\sigma\), the trace energy
\[
E_\sigma
=
\operatorname{tr}(\Pi_\lambda\Pi_\sigma)
\]
measures how much of the obstruction space is supported on that stalk.

If the visual chart distortion satisfies
\[
K>K_{\mathrm{tabular}},
\]
for example \(K_{\mathrm{tabular}}=1.5\), the system shows the tabular contradiction report, but is not admissible as clinical assertion.  This is a useful tool for certified adjudicators, clinical decision makers, and policy makers alone---not a definitive statement of what should or should not occur.

Thus, the decision path is
\begin{alignat*}{1}
&\text{candidate claim} \\
&\quad \color{gray}\xrightarrow{\hspace*{1cm}}\text{\color{black}identity, terminology, provenance, and policy checks} \\
&\qquad\qquad \color{gray}\xrightarrow{\hspace*{1cm}}\text{\color{black}mapping-cone contradiction report} \\
&\qquad\qquad\qquad\qquad \color{gray}\xrightarrow{\hspace*{1cm}}\text{\color{black}human adjudication or abstention} \\
&\qquad\qquad\qquad\qquad\qquad\qquad \color{gray}\xrightarrow{\hspace*{1cm}}\text{\color{black}FHIR transaction only if all gates pass.}
\end{alignat*}

The chart is advisory. The tabular contradiction report is authoritative for review.

\section{Decentralized Cryptographic Sheaf-Enclave protocol}

A centralized permit-before-action safety-receipt layer is one relevant comparison point for fail-closed clinical AI integration. U.S. Patent 12,633,414 B1 describes a safety-receipt layer interposed between an EHR and AI-assisted clinical decision support, with a canonical context envelope, policy evaluation, permit outcomes, and receipt generation \cite{lee_safety_receipt_patent}. The discussion here is a high-level technical comparison only; it is not a claim chart, a legal opinion, or a conclusion about patent scope, validity, infringement, or comparative novelty.

The DCSE design explored here studies a different systems decomposition: distributed non-equivocation/consensus, trusted-execution checks, and sheaf-based consistency certificates are composed at the admission boundary before governed write-back. The contribution claimed in this manuscript is the architecture and its explicit separation of model proposal from deterministic admission checks, not a legal conclusion that the design is patentably novel over any cited reference.

The DCSE protocol sketch contains the following interconnected components:

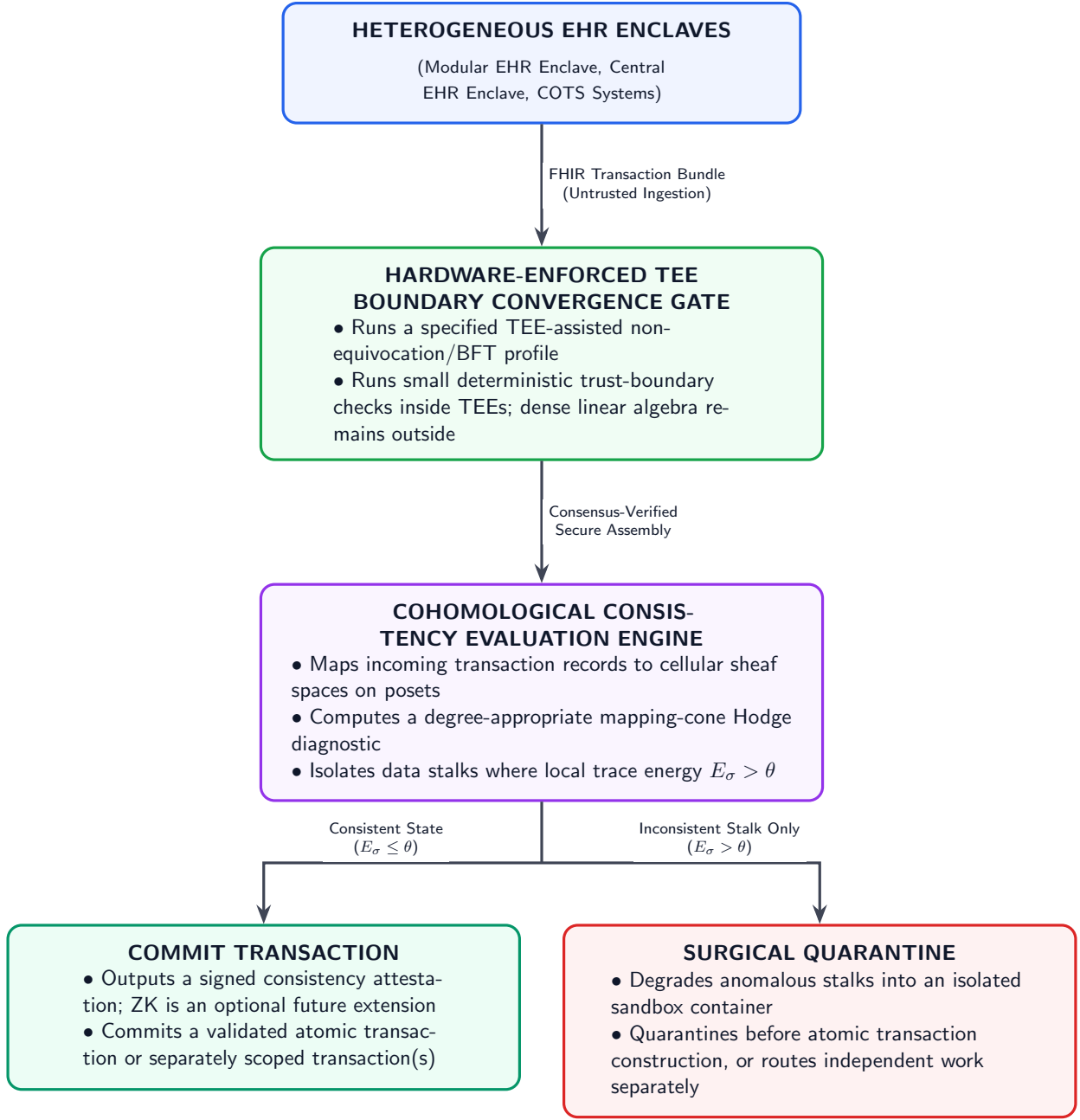
\begin{center} 
\begin{tikzpicture}[node distance=1.0cm and 0.6cm, scale=0.85, transform shape] 
  
  \node[enclave, align=center] (EHR) { 
    \textbf{\uppercase{Heterogeneous EHR Enclaves}}\\[0.15cm] 
    {\footnotesize (Modular EHR Enclave, Central EHR Enclave, COTS Systems)} 
  }; 

  \node[tee, align=center, below=2.2cm of EHR] (TEE) { 
    \textbf{\uppercase{Hardware-Enforced TEE Boundary Convergence Gate}}\\[0.15cm] 
    \begin{minipage}{7.5cm} 
      \textbullet\ Runs a specified TEE-assisted non-equivocation/BFT profile\\ 
      \textbullet\ Runs small deterministic trust-boundary checks inside TEEs; dense linear algebra remains outside 
    \end{minipage}
  }; 

  \node[cohomology, align=center, below=2.2cm of TEE] (COH) { 
    \textbf{\uppercase{Cohomological Consistency Evaluation Engine}}\\[0.15cm] 
    \begin{minipage}{9.0cm}
      \textbullet\ Maps incoming transaction records to cellular sheaf spaces on posets\\ 
      \textbullet\ Computes a degree-appropriate mapping-cone Hodge diagnostic\\ 
      \textbullet\ Isolates data stalks where local trace energy $E_\sigma > \theta$ 
    \end{minipage}
  }; 

  \node[commit, align=center, text width=8.5cm, below left=2.2cm and 5.0cm of COH.south, anchor=north] (COMMIT) { 
    \textbf{\uppercase{Commit Transaction}}\\[0.15cm] 
    \begin{minipage}{6.5cm}
      \textbullet\ Outputs a signed consistency attestation; ZK is an optional future extension\\ 
      \textbullet\ Commits a validated atomic transaction or separately scoped transaction(s) 
    \end{minipage}
  }; 

  \node[quarantine, align=center, text width=8.5cm, below right=2.2cm and 5.0cm of COH.south, anchor=north] (QUARANTINE) { 
    \textbf{\uppercase{Surgical Quarantine}}\\[0.15cm] 
    \begin{minipage}{7.5cm}
      \textbullet\ Degrades anomalous stalks into an isolated sandbox container\\ 
      \textbullet\ Quarantines before atomic transaction construction, or routes independent work separately 
    \end{minipage}
  }; 

  \draw[arrow] (EHR) -- node[labelnode, right] {FHIR Transaction Bundle\\(Untrusted Ingestion)} (TEE); 
  \draw[arrow] (TEE) -- node[labelnode, right] {Consensus-Verified\\Secure Assembly} (COH); 

  \coordinate (split) at ($(COH.south) + (0,-1.1)$); 
  \draw[line] (COH.south) -- (split); 
  \draw[arrow] (split) -| (COMMIT.north) node[pos=0.28, above, labelnode] {Consistent State\\($E_\sigma \le \theta$)}; 
  \draw[arrow] (split) -| (QUARANTINE.north) node[pos=0.32, above, labelnode] {Inconsistent Stalk Only\\($E_\sigma > \theta$)}; 
\end{tikzpicture} 
\end{center}

\begin{enumerate} 
\item \textbf{Hardware-Enforced, TEE-Assisted Consensus}: \\ 
DCSE places small, deterministic trust-boundary checks inside hardware-secured TEEs at participating nodes \cite{intel_sgx,intel_sgx_attestation}. An optimistic TEE-assisted non-equivocation profile may reduce protocol overhead \cite{minbft}, while the conservative fallback uses classical $3f+1$-style BFT assumptions \cite{pbft,bedrock_bft}. The manuscript does not claim a generic $2f+1$ resilience threshold or sub-millisecond latency without a specified protocol, threat model, and benchmark. Dense SVD, QR, and eigensolvers remain outside the enclave as described below. 

\item \textbf{Homological Inconsistency Gating}: \\ 
Rather than executing only static policy-graph rules, DCSE can compile incoming candidate clinical updates into a topological poset representing a cellular sheaf. The mathematical design defines a mapping-cone complex and degree-appropriate Hodge Laplacian; dense construction or eigensolution may occur outside the enclave, with the enclave verifying sparse residuals or signed certificates. The toy numerical appendix uses a projector-based obstruction surrogate, which must not be identified with a mapping-cone Laplacian unless it is constructed from the cone differential. Localized obstruction support is summarized by the basis-invariant trace cell energy: 
\[ 
E_{\sigma} = \text{tr}(\Pi_{\lambda}\Pi_{\sigma}) 
\] 
where $\Pi_{\lambda}$ is the projector onto the null space (representing cohomological obstructions) and $\Pi_{\sigma}$ is the stalk projector. 

\item \textbf{Stalk-Level Surgical Degradation (Algebraic Grafting)}: \\ 
DCSE proposes stalk-level quarantine before constructing an atomic FHIR transaction, or across separately scoped transactions. When $E_{\sigma}$ exceeds the quarantine threshold, the affected stalk is isolated from the candidate write set. If a FHIR Bundle is submitted with \texttt{type=transaction}, its atomic semantics are preserved: the bundle commits as a whole or rolls back as a whole \cite{fhir_r4}. Quarantined data are excluded or routed separately rather than partially committed from the same transaction. This manuscript does not characterize the patent's claim scope as requiring whole-transaction rollback without claim-level analysis. 

\item \textbf{Zero-Knowledge Consistency Attestation}: \\ 
A future deployment may attach a zero-knowledge proof to a hardware-attested consistency result. Such a proof would certify only the specified finite computation and policy predicate without revealing the private witness; it would not assert medical truth. 
\end{enumerate}

\paragraph{Technical distinctions in the design space.}
The following points describe architectural differences rather than legal novelty:
\begin{itemize}
\item \textit{Centralized receipt gating vs. distributed verification}: DCSE combines a TEE-assisted non-equivocation/BFT profile with a conservative BFT fallback; exact resilience depends on the selected protocol and threat model \cite{bedrock_bft}.
\item \textit{Policy evaluation vs. consistency diagnostics}: in addition to deterministic policy checks, DCSE proposes cellular-sheaf consistency diagnostics. A deployed implementation must construct the actual mapping-cone differential before calling a derived Laplacian or obstruction space a mapping-cone certificate.
\item \textit{Candidate-work quarantine vs. transaction atomicity}: candidate items may be quarantined before FHIR transaction construction or routed into separately scoped work. A submitted FHIR transaction retains all-or-nothing semantics \cite{fhir_r4}.
\item \textit{Optional privacy-preserving attestation}: zero-knowledge proof machinery is future work and would attest only to a specified finite computation and policy predicate, not to medical truth.
\end{itemize}

\subsection{Protocol objects}

A DCSE node maintains:
\[
  (B,V,\Theta,\mathcal F,\mathcal W,\mathcal L,\mathcal P).
\]
These objects separate model evidence from enterprise authority:
\begin{description}[leftmargin=2.8em,style=nextline]
\item[\(B\): boundary algebra.] The finite typed universe of assertions and actions that a model or agent is permitted to propose. It prevents arbitrary free-form output from silently becoming an authoritative enterprise fact.
\item[\(V\): versioned semantic bundle.] The terminology, ontology, schema, code-system, mapping, message-format, and other controlled-vocabulary versions that define the meaning of a proposal at evaluation time.
\item[\(\Theta\): evidence registry.] The calibrated local categorical evidence state. In the construction used here, this includes hierarchical Dirichlet evidence. A model probability or logit can update evidence only through declared rules; it is not authority by itself.
\item[\(\mathcal F\): candidate/local state.] The proposed local assertions, mappings, records, actions, or work items whose compatibility is being evaluated.
\item[\(\mathcal W\): authoritative grounding state.] Version-pinned systems of record, signed source data, approved registries, trusted references, human adjudications, and other external evidence against which the candidate state is judged.
\item[\(\mathcal L\): identity/provenance ledger.] The immutable history used to bind a proposal to the correct subject, entity, case, or work item and to detect protocol equivocation. Ledger consistency proves provenance and non-equivocation properties, not the truth of an identity match by itself.
\item[\(\mathcal P\): runtime admissibility policy.] The versioned enterprise contract that maps verified facts, failures, dependencies, authority, and operational context to allowed actions such as admit, quarantine, abstain, expert review, or reject.
\end{description}

\subsection{Operational policy contract and domain portability}\label{sec:policy-portability}

The policy object \(\mathcal P\) is the most deployment-specific component. A useful operational decomposition is
\[
\mathcal P=
(\mathcal A,\mathcal G,\mathcal T,\mathcal D,\mathcal H,\mathcal Q,\mathcal R,\Lambda,\mathcal E,\mathcal F_b),
\]
where \(\mathcal A\) is the allowed action set; \(\mathcal G\) contains hard gating predicates; \(\mathcal T\) contains thresholds, freshness limits, and temporal-validity windows; \(\mathcal D\) specifies dependency rules and dependency closure; \(\mathcal H\) specifies human authority, dual-control, or named-review requirements; \(\mathcal Q\) defines quarantine scope; \(\mathcal R\) defines recovery and escalation; \(\Lambda\) specifies liveness and degraded-operation constraints; \(\mathcal E\) defines signed exception and override rules; and \(\mathcal F_b\) defines failover behavior when required services or evidence are unavailable. A policy instance should also be attributable through fields such as
\[
(\texttt{policy\_id},\texttt{version},\texttt{effective\_time},\texttt{authority},\texttt{hash}).
\]
This decomposition is an engineering contract, not an additional mathematical theorem. The invariant is that every admissibility decision is reproducible under an explicit policy version rather than inferred ad hoc from model confidence.

The verification architecture is therefore portable while the domain semantics are not. Table~\ref{tab:domain-portability} illustrates how the same tuple can be instantiated in three enterprise settings without claiming that a healthcare benchmark can be copied unchanged into another sector.

\begin{table}[htbp]
\centering
\scriptsize
\caption{Illustrative domain translations of the GBI/DCSE protocol objects. The architecture is shared; semantics, authoritative evidence, and policy are domain-specific.}
\label{tab:domain-portability}
\begin{tabularx}{\textwidth}{@{}lXXX@{}}
\toprule
Object & Healthcare / EHR & Financial services & Government / mission systems \\
\midrule
\(B\) & Typed mapping, identity, temporal, evidence, and admissible-action states. & Approve, hold, review, sanctions-review, fraud-review, reconciliation-conflict, reject. & Verified, conflicting, incomplete, releasable, restricted, review-required, reject. \\ \addlinespace
\(V\) & FHIR versions; RxNorm, SNOMED CT, LOINC, local dictionaries; policy versions. & Instrument and product taxonomies; legal-entity schemas; message, sanctions-list, and risk-model versions. & Mission schemas; controlled vocabularies; policy directives; handling and data-standard versions. \\ \addlinespace
\(\Theta\) & Evidence over bounded mapping, identity, temporal, or policy-action categories. & Evidence over legitimate/suspicious/blocked/unknown, entity matching, or other bounded risk decisions. & Evidence over entity resolution, source reliability, mission status, or action categories. \\ \addlinespace
\(\mathcal F\) & Candidate mapping, FHIR resource, identity link, temporal status, or write-ready work item. & Candidate payment, trade, reconciliation update, underwriting result, account change, or agent action. & Candidate claim, case update, entity association, report element, workflow action, or agent tool call. \\ \addlinespace
\(\mathcal W\) & Source EHR records, authoritative terminology, provenance, policy, and human adjudication. & Core ledger, KYC systems, approved market feeds, sanctions sources, risk limits, signed authorizations. & Authoritative registries, signed source records, mission systems, directives, and human adjudications. \\ \addlinespace
\(\mathcal L\) & Patient/entity identity decisions plus immutable execution provenance. & Customer, account, and legal-entity identity plus transaction provenance and authorization lineage. & Entity, case, and credential identity plus signed provenance and non-equivocation history. \\ \addlinespace
\(\mathcal P\) & Identity, terminology, temporal, evidence, dependency, write-policy, review, and failover constraints. & Limits, counterparty eligibility, sanctions, freshness, separation of duties, jurisdiction, human approval, permitted automation, overrides, and failover. & Access, classification/handling, need-to-know, provenance, jurisdiction, autonomous-action limits, human authority, continuity-of-operations, and degraded-mode rules. \\
\bottomrule
\end{tabularx}
\end{table}

\paragraph{Liveness without violating atomicity.}
Localized quarantine is a pre-commit work-item mechanism, not permission to partially commit an operation whose underlying system requires atomicity. In a financial workflow, for example, one disputed payment or legal-entity mapping can be excluded from a candidate settlement batch and routed for review while independent verified work continues if institutional policy permits. In a government workflow, one ambiguous entity association can be quarantined while unrelated mission work remains live. The same rule appears in the EHR setting: an inadmissible stalk is excluded or separately routed before an atomic FHIR transaction is constructed.

\paragraph{Human review and automation bias.}
The review surface should expose the deterministic reasons for inadmissibility---for example candidate value, authoritative evidence, source freshness, schema or terminology version, authority, dependencies, provenance, policy rule, and required action---rather than reduce review to a model confidence number. Geometric charts can remain advisory, while a scannable tabular contradiction record is the authoritative review artifact.

\paragraph{Customer-deployed diagnostic evaluation.}
A versioned evaluation package can be executed locally against an enterprise's actual model, agent, retrieval, data-interface, and policy stack. The output is then a customer-specific empirical failure surface rather than only a generic leaderboard score or proxy estimate. Failure slices can inform a data-development roadmap: expert labeling, synthetic or counterfactual examples, evaluator and grader development, retrieval changes, benchmark expansion, or post-training environments, followed by re-evaluation. This diagnostic loop is vendor-neutral; it identifies where intervention may add value without predetermining the intervention or provider.

\subsection{TEE execution and sparse verification}

Trusted execution environments are useful for confidentiality and attestation \cite{intel_sgx,intel_sgx_attestation}, but enclaves are not a place to run arbitrary dense numerical workloads. The DCSE design places small, deterministic checks inside the enclave:
\begin{itemize}
\item signature and nonce verification,
\item sparse residual checks such as \(\|Lx\|\),
\item hash-chain validation,
\item policy evaluation,
\item certificate signing.
\end{itemize}
Dense SVD, QR, and large eigensolvers are performed outside the enclave or replaced with sparse certified residual checks. A WASM/WAMR-style packaging can enforce heap, I/O, and syscall budgets for deterministic enclave modules.

\subsection{BFT identity logging and fallback}

TEE-assisted protocols using a unique sequential identifier generator can reduce communication \cite{minbft} in optimistic settings, but hardware counters can fail. Therefore the identity ledger has two modes.

\paragraph{Fast path.} If attestation and supported monotonic-counter or equivalent trusted freshness services are healthy, the node may use a TEE-assisted non-equivocation profile \cite{intel_sgx_attestation,intel_sgx_pse}.

\paragraph{Fallback path.} On counter failure, enclave restart, attestation loss, or equivocation evidence, authoritative clinical writes halt and the system falls back to conservative \(3f+1\)-style BFT assumptions \cite{pbft, bedrock_bft}. The fallback may be implemented by a mechanized BFT protocol family \cite{zhao_bythos,bythos_artifact}. This protects clinical safety from a single hardware-counter failure.

The ledger proves protocol-level non-equivocation, not patient identity truth. Clinical identity remains an external validity predicate.

\subsection{Zero-knowledge consistency attestations}

A future deployment may attach a zero-knowledge proof that a specified finite verification computation and policy predicate passed without revealing the private witness. A schematic statement is
\[
  \exists\,w:\
  \operatorname{VerifyConeCertificate}(c,w)=1
  \quad\land\quad
  \operatorname{Policy}(w)=\texttt{pass},
\]
where the public input $c$ binds the transaction digest, policy version, and verification parameters. If vanishing of a particular cohomology group is part of the certificate, that property must be encoded by a concrete finite computation or residual check; the expression $H^1(\operatorname{Cone}\varphi)=0$ is not itself a witness predicate. The proof would attest only to the encoded computation, not to medical truth.

\section{Complete EHR interoperability application at scale}

The motivating setting is a large federated clinical network with legacy RPMS/VistA-like records, local site dictionaries, terminology services, FHIR-capable endpoints, identity-reconciliation needs, and human review \cite{ihs_modernization,ihs_rpms}. The Indian Health Service (IHS) is used here as a public interoperability context, not as a deployment claim, endorsement, or statement that this architecture is part of PATH EHR. The proposed substrate would sit between systems as a fail-closed co-processor rather than replace the EHR.

The complete pipeline for this case would be:

\begin{enumerate}
\item \textbf{Capability discovery.} Query the destination FHIR server's capability statement and record supported resources, interactions, versioning, security, and transaction behavior \cite{fhir_r4}.
\item \textbf{Identity gate.} Resolve source and destination patient identifiers through a BFT-backed identity-decision log. If identity confidence or protocol safety fails, halt.
\item \textbf{Discovery parse.} An LLM or parser emits logit receipts over local categories, such as medication mapping status or allergy state.
\item \textbf{Boundary type check.} The judgment engine checks category sets, timestamps, facility scope, terminology bundle, and provenance.
\item \textbf{Evidence update.} Accepted receipts update local hierarchical Dirichlet registries.
\item \textbf{Sheaf diagnostic.} Mapping-cone stalk energies identify contradictions.
\item \textbf{Reviewer presentation.} High-risk contradictions are shown as tables first. Charts are optional secondary aids.
\item \textbf{Commit or abstain.} Safe human-approved changes are sent as FHIR transactions with Provenance and AuditEvent. Failures return OperationOutcome and quarantine data \cite{fhir_r4,fhir_provenance,fhir_outcome}.
\end{enumerate}

\subsection{Example: penicillin allergy and amoxicillin}

Suppose the source record contains an active confirmed high-criticality penicillin allergy. A clinician proposes amoxicillin \cite{fda_amoxil}. The parser emits a medication-allergy logit receipt over
\[
\{\texttt{no-conflict},\texttt{possible-conflict},\texttt{confirmed-conflict},\texttt{unknown}\}.
\]
The judgment engine checks a version-pinned clinical terminology and policy bundle; if that bundle classifies the proposed order as conflicting with the recorded allergy state and the sheaf diagnostic exceeds its review threshold, the system treats the candidate write as inadmissible. It emits \cite{fhir_outcome}:
\[
\texttt{OperationOutcome(severity=error, code=business-rule)}.
\]
No medication transaction is committed unless a named clinician performs an explicit override with Provenance \cite{fhir_provenance}.

\subsection{Example: metformin without renal context \cite{metformin_label}}

Suppose a metformin order is proposed and the version-pinned local clinical policy requires qualifying renal context, but no qualifying renal observation is available. The local category is
\[
\{\texttt{renal-context-present},\texttt{renal-context-expired},\texttt{renal-context-missing}\}.
\]
Under that version-pinned policy, a high posterior mass on \texttt{renal-context-missing} routes the candidate medication request to quarantine or review rather than authorizing a write. The system may commit a Provenance record of the abstention but not the medication update.

\subsection{FHIR resources used}

The core implementation uses \cite{fhir_r4,fhir_provenance,fhir_outcome}:
\[
\begin{array}{ll}
\texttt{Patient} & \text{identity-bound subject},\\
\texttt{MedicationRequest} & \text{candidate medication orders},\\
\texttt{AllergyIntolerance} & \text{allergy state},\\
\texttt{Observation} & \text{labs such as eGFR},\\
\texttt{Provenance} & \text{evidence lineage and signatures},\\
\texttt{AuditEvent} & \text{security audit trail},\\
\texttt{OperationOutcome} & \text{structured fail-closed response}.
\end{array}
\]

\section{Empirical admission-boundary evaluation: GBI BoundaryBench v0.1}\label{sec:boundarybench}

The architecture above motivates a separate empirical question: can a model proposal cross a frozen structured-admissibility boundary without access to the trusted answer package? GBI BoundaryBench v0.1 is a synthetic legacy-EHR benchmark built for that purpose \cite{boundarybench_repo}. It is a companion evaluation artifact, not a clinical validation study of DCSE.

The frozen held-out package contains 256 tasks spanning eight task families. A single open-weight model, \texttt{Qwen/Qwen3-4B-Instruct-2507}, was evaluated under three preregistered evidence modes: \texttt{output\_only}, \texttt{token\_top\_k}, and \texttt{full\_category\_evidence}. The model-execution host contained the answer-key-free model inputs and cryptographically pinned manifests, but not the trusted held-out references. Raw outputs were frozen before trusted scoring.

Across the three canonical runs there were $256\times3=768$ completed executions. The frozen scorer accepted zero model result records. The status distribution was
\[
369\ \texttt{safe\_parse\_reject}
\qquad\text{and}\qquad
399\ \texttt{safe\_schema\_reject},
\]
with the same per-mode split of 123 parse rejects and 133 schema rejects. Thus coverage was $0$, the invalid-output rate was $1.0$, all 768 executions were quarantined by the benchmark policy, and selective risk is undefined because no output entered the accepted set.

This is a result about the interface and one frozen model/configuration, not a general statement about LLM capability, healthcare safety, or the usefulness of richer evidence in other settings. In this run, additional model-side evidence access did not improve verified completion because failures occurred earlier at the parse/schema admissibility boundary. The experiment therefore illustrates the system-level distinction emphasized throughout this paper: successful neural inference is not equivalent to an admissible downstream action.

The public repository releases the architecture, aggregate scored results, figures, and provenance hashes while withholding hidden held-out references and raw held-out responses so that future blind evaluations remain possible \cite{boundarybench_repo}.

\section{Summary: hallucination containment and external validity}

The precise safety claim of the framework is not that a neural model is prevented from hallucinating internally. A model may still produce a high-confidence but unsupported
proposal. The claim precisely is:

\[
\boxed{
  \begin{tabular}{c}
    The framework prevents unsupported model outputs from being silently promoted into \\
    authoritative clinical facts or EHR writes.
  \end{tabular}
}
\]

Thus, the architecture provides \emph{hallucination containment and admission control}, though not hallucination elimination.

\subsection{Where hallucination is blocked}

The system has several gates. A hallucination may occur at the LLM or logit layer, but it must pass deterministic validity checks before it can affect clinical state.

\subsubsection{The logit layer is treated as evidence, not authority}

The model emits logits
\[
L(h)=Wh+b, \qquad p=\operatorname{softmax}(L/\tau).
\]

The output is not interpreted as:
\[
\text{``the model says this, therefore it is true.''}
\]
It is interpreted as:
\[
\text{``the model proposes a categorical evidence vector over a bounded decision.''}
\]

A safe logit receipt should include
\[
(L_1,\ldots,L_K), \qquad p_i=\frac{\exp(L_i/\tau)}{\sum_j \exp(L_j/\tau)}, \qquad H(p), \qquad \Delta=L_{(1)}-L_{(2)}, \qquad D_{\mathrm{KL}}(p_{\mathrm{full}}\|p_{\mathrm{truncated}}).
\]

High entropy, small margin, large tail loss, or cross-model disagreement means
\[
\text{do not commit; route to review.}
\]

The logit layer can expose uncertainty and instability, but it does not establish admissibility or external validity.

\subsubsection{The boundary algebra restricts what the model is allowed to say}

The model cannot emit arbitrary prose directly as an authoritative clinical assertion. It must map into a finite boundary
decision such as:
\[
\{
\text{exact},
\text{equivalent},
\text{narrower},
\text{broader},
\text{conflict},
\text{unmapped}
\}
\]
or,
\[
\{
\text{no allergy documented},
\text{confirmed allergy},
\text{unconfirmed allergy},
\text{refuted}
\}.
\]

This turns free-form generation into a typed proposal over a small categorical interface. Thus the first defense is:

\[
\boxed{
\text{No free-form clinical assertion can directly become a database update.}
}
\]

\subsubsection{The model proposal is checked against external validity}

The operational admission gate is the external-validity predicate
\[
EV:B\times W\to\{0,1\},
\]
where \(B\) is the finite boundary algebra and \(W\) is the measurable world/evidence
space.

Operationally, \(EV\) is a terminating Boolean procedure over authoritative evidence, not an
LLM judgment.

In the EHR setting,
\[
W =
\left(
\begin{array}{l}
\text{patient identity log},\\
\text{FHIR resources},\\
\text{signed terminology bundle},\\
\text{Provenance},\\
\text{AuditEvent},\\
\text{policy rules},\\
\text{human adjudications}
\end{array}
\right).
\]

A clinical claim \(b\in B\) is accepted only if
\[
EV(b,w)=1.
\]

This is an operational admissibility test relative to an institutionally defined evidence model; it is not a declaration of ground-truth medical certainty.

A more explicit semantics is:
\[
[[b]]_{\mathcal M}
\subseteq
\mathcal W\times\mathcal T\times\mathcal L\times\mathcal V\times\mathcal P,
\]
where,
\[
\mathcal W=\text{clinical evidence state},
\]
\[
\mathcal T=\text{time interval},
\]
\[
\mathcal L=\text{care setting, facility, or source system},
\]
\[
\mathcal V=\text{terminology version},
\]
and,
\[
\mathcal P=\text{provenance and authority chain}.
\]

The system accepts a proposed assertion only when its witness lies inside the interpretation:
\[
(w,t,\ell,v,p)\in [[b]]_{\mathcal M}.
\]

\subsection{What kind of hallucination this prevents}

The framework prevents many operationally dangerous hallucinations from becoming authoritative clinical state.

\begin{table}[h]
\centering
\begin{tabular}{p{0.36\linewidth}p{0.54\linewidth}}
\hline
\textbf{Hallucination type} & \textbf{How it is caught} \\
\hline
Nonexistent code &
Signed terminology bundle rejects it. \\ \addlinespace

Unsupported clinical assertion &
No matching FHIR resource or provenance witness exists, so \(EV=0\). \\ \addlinespace

Wrong patient &
The external identity predicate must reject an invalid or ambiguous match; the BFT log provides non-equivocation and auditability of the identity decision, not correctness of the match itself. \\ \addlinespace

Stale fact &
Temporal interval check fails. \\ \addlinespace

Free-text overreach &
NLP output lacks acceptable provenance or human signoff. \\ \addlinespace

Unsupported medication mapping &
Categorical mapping remains \texttt{unmapped}, \texttt{conflict}, or \texttt{unknown}. \\ \addlinespace

Internal inconsistency &
Sheaf or mapping-cone obstruction flags incompatible local sections. \\ \addlinespace

Overconfident model guess &
Entropy, margin, tail-loss, or cross-model checks trigger review. \\ \addlinespace

Unsafe commit &
FHIR transaction is not posted; an \texttt{OperationOutcome} and Provenance record the abstention. \\
\hline
\end{tabular}
\caption{Operational hallucination containment mechanisms.}
\end{table}

The central design invariant is:

\[
\boxed{
\text{The substrate emits typed certificates, not autonomous clinical decisions.}
}
\]

On precondition failure, it emits an \texttt{OperationOutcome}, records Provenance, and refuses authoritative write-back.

\subsection{What kind of hallucination it does not prevent}

The framework does not prevent a model from internally forming a false but coherent story. A model can still produce
\[
p(\text{equivalent mapping})=0.98
\]
for a bad mapping.

A model can still produce a fluent explanation. A model can still be wrong in a way that is internally consistent.

The architecture blocks that proposal only if the claim fails some external test:
\[
EV(b,w)=0,
\]
or creates a sheaf obstruction,
\[
[ds]\ne 0,
\]
or lacks provenance,
\[
\operatorname{Prov}(b)=\varnothing,
\]
or violates policy,
\[
\operatorname{Policy}(b)=0.
\]

Therefore,

\begin{center}
\fbox{\parbox{0.88\linewidth}{\centering
GBI/DCSE is not hallucination elimination. It is hallucination containment and admission control.}}
\end{center}

The following two examples are synthetic policy-gating illustrations, not medication recommendations or clinical decision support.

\subsection{Example: penicillin allergy}

Suppose an LLM proposes:
\[
\text{``Amoxicillin is safe for this patient.''}
\]

The framework does not ask whether the sentence sounds plausible. It atomizes the proposal into a boundary claim
\[
b=
\text{``amoxicillin order is compatible with active allergy profile.''}
\]

It then checks the authoritative state
\[
w =
\left\{
\begin{array}{l}
\text{Patient X identity accepted},\\
\text{AllergyIntolerance: penicillin allergy},\\
\text{clinicalStatus=active},\\
\text{verificationStatus=confirmed},\\
\text{criticality=high},\\
\text{RxNorm/SNOMED terminology bundle}
\end{array}
\right\} \cite{fda_amoxil}.
\]

If the sheaf consistency check finds conflict on the allergy--medication compatibility relation, the substrate emits a \texttt{business-rule} \texttt{OperationOutcome} and commits no clinical transaction.

The hallucination is not prevented at generation time. It is prevented from becoming a clinical write.

\subsection{Example: metformin without renal context \cite{metformin_label}}

Suppose the model says:
\[
\text{``The proposed metformin order is appropriate.''}
\]

The system asks:
\[
EV(\text{metformin order allowed},w)=1?
\]

If the version-pinned clinical policy requires a qualifying renal-function observation \cite{metformin_label} and no observation satisfying that policy's terminology and validity window is available, the necessary witness is missing.

The system emits a warning or required \texttt{OperationOutcome}, records the abstention, and does not commit the medication order.  Absence of evidence is not silently filled by the model.

\subsection{Admissibility as a layered operational criterion}

The substrate's acceptance criterion is operational:

\[
\operatorname{Admissible}_{\mathrm{GBI}}(b)
=
1
\quad\Longleftrightarrow\quad
\text{there exists an admissible witness }w\text{ such that }EV(b,w)=1.
\]

An admissible witness is not a logit. It is a bundle
\[
w=
\left(
\begin{array}{l}
\text{identity certificate},\\
\text{FHIR resource},\\
\text{terminology version},\\
\text{provenance signature},\\
\text{temporal interval},\\
\text{policy rule},\\
\text{optional human adjudication}
\end{array}
\right).
\]

The layers have distinct roles:
\[
\text{logits propose},
\]
\[
\text{Dirichlet evidence calibrates},
\]
\[
\text{sheaves check consistency},
\]
\[
\text{BFT checks non-equivocation},
\]
\[
\text{FHIR and Provenance check source authority},
\]
\[
\text{human review handles residual ambiguity}.
\]

No single layer establishes medical truth. The operational admission decision is the composition of external-validity predicates over institutionally authoritative records and policy.

\paragraph{The Proposal.}
The system does not claim to eliminate hallucinations inside neural models. A neural model may emit a high-confidence but unsupported proposal. The GBI/DCSE architecture prevents such proposals from becoming authoritative clinical state unless they are witnessed by the external validity predicate over signed identity, terminology, provenance, temporal, and policy objects. In this sense, the architecture provides hallucination containment: it converts unsupported model outputs into abstentions, review tasks, or \texttt{OperationOutcome} records, rather than FHIR commits.

\paragraph{Model-relative admissibility.}
Admissibility is model-relative. A boundary assertion \(b\) is admissible to the substrate only relative to a declared institutional model \(\mathcal M\) consisting of accepted identity facts, version-pinned
terminology, FHIR resources, temporal validity windows, provenance signatures, policy rules, and human adjudications. The substrate can be wrong if \(\mathcal M\) is wrong; it cannot repair corrupted source records or incomplete institutional policy. Its guarantee is traceability of accepted claims to declared admissible witnesses and policies, not that those witnesses perfectly mirror the world.

\appendix

\section{Complete, Executable Reference Implementation in Julia}\label{app:code}
The executable source is reproduced below and is also supplied as arXiv ancillary material in \texttt{anc/gbi\_dcse\_arxiv\_revised.jl}. Run it directly in a Julia environment. It requires only the standard \texttt{LinearAlgebra} and \texttt{Printf} libraries. The appendix implements numerical certificates and a toy projector-based stalk-obstruction surrogate; it is not the enterprise policy object \(\mathcal P\), a FHIR writer, an enclave, or a clinical decision system.

\begin{lstlisting}[language=Julia]
using LinearAlgebra
using Printf

# Numerical appendix for the DCSE/GBI paper.
#
# The paper treats neural/logit output as bounded evidence, not clinical truth.
# This script verifies the small numerical certificates used by that boundary:
# Dirichlet-Fisher conditioning, advisory chart distortion,
# higher-dimensional ellipsoid distortion, and a projector-based stalk-obstruction surrogate.
#
# It is not a FHIR writer, enterprise policy engine, enclave, or clinician.
# QUARANTINE_THRESHOLD is a toy appendix parameter, not the runtime policy P.
# The only commit/quarantine decision represented here is the toy stalk-level
# trace-energy gate used to illustrate localization.

const MAPPING_STATUS_LABELS = ("exact", "equivalent", "narrower", "broader", "conflict", "unmapped")
const FISHER_MIN_EIGENVALUE = 1e-6
const QUARANTINE_THRESHOLD = 0.80
const KTABULAR = 1.5
const PSD_TOL = 1e-10

function condensed_probe_demo()
    # Section 4 intuition: a condensed/operational probe is stricter than
    # ordinary pointwise convergence. The toy sequence 1/n converges, but it is
    # not eventually constant, so it fails the operational compatibility check.
    seq = [1.0 / i for i in 1:20]
    tail_variation = maximum(abs.(seq[end-4:end] .- seq[end]))
    return (
        eventually_constant = tail_variation < 1e-12,
        tail_variation = tail_variation,
    )
end

function entropy_bits(p::AbstractVector{<:Real})::Float64
    entropy = 0.0
    @inbounds for x in p
        xf = Float64(x)
        if xf > 0.0
            entropy -= xf * log2(xf)
        end
    end
    return entropy
end

function approx_trigamma(x::Float64)::Float64
    # Section 6 uses the Dirichlet Fisher metric
    # g_ij(alpha) = trigamma(alpha_i) * delta_ij - trigamma(sum(alpha)).
    # This recurrence plus asymptotic expansion avoids non-stdlib packages.
    x > 0.0 || throw(DomainError(x, "trigamma approximation requires x > 0"))

    y = x
    acc = 0.0
    while y < 8.0
        acc += 1.0 / (y * y)
        y += 1.0
    end

    inv_y = 1.0 / y
    inv2 = inv_y * inv_y
    inv3 = inv2 * inv_y
    inv5 = inv2 * inv3
    inv7 = inv2 * inv5
    inv9 = inv2 * inv7

    return acc + inv_y + 0.5 * inv2 + inv3 / 6.0 - inv5 / 30.0 +
           inv7 / 42.0 - inv9 / 30.0
end

function fisher_dirichlet(alpha::AbstractVector{<:Real})::Matrix{Float64}
    # The caller is expected to enforce the paper's evidence box alpha_i in
    # [phi, A]. This routine still rejects non-positive values so the metric
    # cannot be evaluated on the singular simplex boundary.
    k = length(alpha)
    k > 0 || throw(ArgumentError("alpha must be non-empty"))

    alpha_sum = 0.0
    @inbounds for a in alpha
        af = Float64(a)
        af > 0.0 || throw(DomainError(af, "Dirichlet alpha values must be > 0"))
        alpha_sum += af
    end

    base = -approx_trigamma(alpha_sum)
    metric = fill(base, k, k)
    @inbounds for i in eachindex(alpha)
        metric[i, i] += approx_trigamma(Float64(alpha[i]))
    end
    return metric
end

f_map(z::ComplexF64, theta::Float64)::ComplexF64 = z * z + theta * conj(z)

function wirtinger_exact(z::ComplexF64, theta::Float64)
    # Section 9 says the visual chart is advisory only. For the fixed chart
    # f(z; theta) = z^2 + theta * conj(z), the Wirtinger derivatives are exact:
    # df/dz = 2z and df/dconj(z) = theta.
    return 2.0 * z, ComplexF64(theta, 0.0)
end

function wirtinger_fd(z::ComplexF64, theta::Float64, h::Float64)
    # Kept as a validation fallback for the exact derivative above.
    fx = (f_map(z + h, theta) - f_map(z - h, theta)) / (2.0 * h)
    fy = (f_map(z + im * h, theta) - f_map(z - im * h, theta)) / (2.0 * h)
    return 0.5 * (fx - im * fy), 0.5 * (fx + im * fy)
end

function hyperellipsoid_certificate(A::AbstractMatrix{<:Real})
    # Section 8.1 certificate: Df maps infinitesimal spheres to ellipsoids.
    # Singular values provide H = sigma_max / sigma_min and
    # K_outer = sigma_max^n / det(A). Positive Jacobian is required.
    A64 = Matrix{Float64}(A)
    singular_values = svdvals(A64)
    sigma_min = minimum(singular_values)
    sigma_max = maximum(singular_values)
    jacobian = det(A64)

    sigma_min > 0.0 || throw(DomainError(sigma_min, "matrix must be full rank"))
    jacobian > 0.0 || throw(DomainError(jacobian, "matrix must have positive Jacobian"))

    return (
        singular_values = singular_values,
        axis_eccentricity = sigma_max / sigma_min,
        jacobian = jacobian,
        outer_distortion = sigma_max^size(A64, 1) / jacobian,
    )
end

function orthonormal_columns(M::AbstractMatrix{<:Real})::Matrix{Float64}
    # The trace energy E_sigma = tr(P_lambda P_sigma) is basis-invariant only
    # after the obstruction basis has been made orthonormal.
    M64 = Matrix{Float64}(M)
    rows, cols = size(M64)
    cols <= rows || throw(ArgumentError("matrix must have at least as many rows as columns"))

    factor = qr(M64)
    return Matrix(factor.Q)[:, 1:cols]
end

function appendix_raw_obstructions()
    # Appendix A's explicit 6x2 obstruction matrix.
    #
    # The row pairs model Allergy, MedicationRequest, and RenalLab stalks. The
    # After orthonormalization, this explicit matrix reproduces the Section 7.3
    # toy stalk energies to the displayed rounding. The routine below remains a
    # projector-based surrogate; it does not construct a mapping-cone differential.
    return [0.10 0.10;
            0.10 -0.10;
            1.00 1.00;
            1.00 -1.00;
            0.05 0.00;
            0.00 0.05]
end

function mapping_cone_certificate(
    raw_obstructions::AbstractMatrix{<:Real};
    threshold::Float64 = QUARANTINE_THRESHOLD,
)
    # Toy obstruction surrogate used by Sections 7-10. P_lambda projects onto
    # the span of the supplied obstruction vectors; each stalk energy is the trace
    # against a coordinate-stalk projector. This is not a mapping-cone Laplacian.
    U = orthonormal_columns(raw_obstructions)
    projector = U * transpose(U)
    laplacian = Matrix{Float64}(I, size(projector, 1), size(projector, 2))
    laplacian .-= projector

    stalk_ranges = (
        ("Allergy", 1:2),
        ("MedicationRequest", 3:4),
        ("RenalLab", 5:6),
    )
    energies = map(stalk_ranges) do (name, idxs)
        energy = sum(projector[i, i] for i in idxs)
        (name = name, energy = energy, decision = energy > threshold ? "QUARANTINE" : "COMMIT")
    end

    return (
        basis = U,
        projector = projector,
        laplacian = laplacian,
        eigenvalues = eigvals(Symmetric(laplacian)),
        energies = energies,
    )
end

function run_self_check()
    alpha = [2.0, 3.0, 4.0, 5.0]
    metric = fisher_dirichlet(alpha)
    @assert issymmetric(metric)
    @assert minimum(eigvals(Symmetric(metric))) > FISHER_MIN_EIGENVALUE

    boundary_alpha = [0.01, 3.0, 4.0, 5.0]
    boundary_metric = fisher_dirichlet(boundary_alpha)
    @assert minimum(eigvals(Symmetric(boundary_metric))) > FISHER_MIN_EIGENVALUE

    z = 1.0 + 1.0im
    theta = 0.35
    fz_exact, fzb_exact = wirtinger_exact(z, theta)
    fz_fd, fzb_fd = wirtinger_fd(z, theta, 1e-5)
    @assert isapprox(fz_exact, fz_fd; rtol = 1e-10, atol = 1e-10)
    @assert isapprox(fzb_exact, fzb_fd; rtol = 1e-10, atol = 1e-10)

    A = [1.20 0.10 0.0; 0.20 0.80 0.05; 0.0 0.10 1.10]
    cert = hyperellipsoid_certificate(A)
    @assert cert.axis_eccentricity >= 1.0
    @assert cert.outer_distortion >= 1.0

    raw = appendix_raw_obstructions()
    cone = mapping_cone_certificate(raw)
    @assert isapprox(cone.basis' * cone.basis, Matrix{Float64}(I, 2, 2); atol = 1e-12)
    @assert issymmetric(cone.laplacian)
    @assert minimum(cone.eigenvalues) >= -PSD_TOL

    angle = 0.73
    rotation = [cos(angle) -sin(angle); sin(angle) cos(angle)]
    rotated_projector = (cone.basis * rotation) * transpose(cone.basis * rotation)
    @assert norm(cone.projector - rotated_projector) <= 1e-12

    return true
end

function run_report()
    println("========================================================================")
    println("        GBI/DCSE MATHEMATICAL CO-PROCESSOR VERIFICATION RUN")
    println("========================================================================")

    probe = condensed_probe_demo()
    println("\n  Condensed operational probe:")
    println("    Sequence 1/n eventually constant? ", probe.eventually_constant)
    @printf("    Tail variation over last five terms: %.6f\n", probe.tail_variation)

    status_alpha = [1.0, 14.0, 1.0, 1.0, 3.0, 1.0]
    status_p = status_alpha ./ sum(status_alpha)
    println("\n  Local mapping-status evidence:")
    println("    Labels: ", MAPPING_STATUS_LABELS)
    println("    Posterior mean p: ", round.(status_p; digits = 4))
    @printf("    Entropy: %.4f bits\n", entropy_bits(status_p))

    alpha_a = [2.0, 3.0, 4.0, 5.0]
    alpha_b = [0.01, 3.0, 4.0, 5.0]
    for (label, alpha) in (("Interior", alpha_a), ("Boundary", alpha_b))
        metric = fisher_dirichlet(alpha)
        ev = eigvals(Symmetric(metric))
        @printf("\n  %s alpha: %s\n", label, alpha)
        println("    Eigenvalues: ", round.(ev; digits = 4))
        @printf("    Condition number: %.2f\n", maximum(ev) / minimum(ev))
    end

    z = 1.0 + 1.0im
    theta = 0.35
    fz, fzb = wirtinger_exact(z, theta)

    A = [1.20 0.10 0.0; 0.20 0.80 0.05; 0.0 0.10 1.10]
    ellipsoid = hyperellipsoid_certificate(A)
    println("\n  Scaling matrix A: ", A)
    println("    Singular values:     ", round.(ellipsoid.singular_values; digits = 4))
    @printf("    Axis eccentricity H: %.4f\n", ellipsoid.axis_eccentricity)
    @printf("    Jacobian det J:      %.4f\n", ellipsoid.jacobian)
    @printf("    Outer distortion Ko: %.4f\n", ellipsoid.outer_distortion)

    raw_obstructions = appendix_raw_obstructions()
    cone = mapping_cone_certificate(raw_obstructions)
    println("\n  Eigenvalues of cone Laplacian Lc: ", round.(cone.eigenvalues; digits = 4))
    println("  Stalk obstruction energies:")
    for item in cone.energies
        @printf("    %-18s energy = %.4f -> %s\n", item.name, item.energy, item.decision)
    end

    println("========================================================================")
end

function trigamma_kernel(values::AbstractVector{Float64})::Float64
    acc = 0.0
    @inbounds for x in values
        acc += approx_trigamma(x)
    end
    return acc
end

function wirtinger_exact_kernel(points::AbstractVector{ComplexF64}, theta::Float64)::Float64
    acc = 0.0
    @inbounds for z in points
        fz, fzb = wirtinger_exact(z, theta)
        acc += abs2(fz) + abs2(fzb)
    end
    return acc
end

function wirtinger_fd_kernel(points::AbstractVector{ComplexF64}, theta::Float64, h::Float64)::Float64
    acc = 0.0
    @inbounds for z in points
        fz, fzb = wirtinger_fd(z, theta, h)
        acc += abs2(fz) + abs2(fzb)
    end
    return acc
end

function benchmark_call(label::String, f::Function, reps::Int; ops_per_rep::Int = 1)
    f()
    GC.gc()
    last_result = nothing
    guarded_f = Base.inferencebarrier(f)
    elapsed = @elapsed begin
        for _ in 1:reps
            last_result = Base.inferencebarrier(guarded_f())
        end
    end
    @printf(
        "%-34s %10.3f us/op  (%d reps, %d ops/rep)\n",
        label,
        elapsed * 1e6 / (reps * ops_per_rep),
        reps,
        ops_per_rep,
    )
    return last_result
end

function run_benchmarks()
    println("Benchmarks exclude Julia startup/compilation cost and use Base timing only.")

    alpha = [2.0, 3.0, 4.0, 5.0]
    theta = 0.35
    A = [1.20 0.10 0.0; 0.20 0.80 0.05; 0.0 0.10 1.10]
    raw = appendix_raw_obstructions()
    trigamma_values = collect(range(0.25, 12.0; length = 512))
    points = ComplexF64[
        ComplexF64(cos(t), sin(t)) for t in range(0.1, 2.9; length = 512)
    ]

    benchmark_call(
        "approx_trigamma",
        () -> trigamma_kernel(trigamma_values),
        10_000;
        ops_per_rep = length(trigamma_values),
    )
    benchmark_call("fisher_dirichlet", () -> fisher_dirichlet(alpha), 200_000)
    benchmark_call(
        "wirtinger_exact",
        () -> wirtinger_exact_kernel(points, theta),
        50_000;
        ops_per_rep = length(points),
    )
    benchmark_call(
        "wirtinger_fd",
        () -> wirtinger_fd_kernel(points, theta, 1e-5),
        50_000;
        ops_per_rep = length(points),
    )
    benchmark_call("hyperellipsoid_certificate", () -> hyperellipsoid_certificate(A), 100_000)
    benchmark_call("mapping_cone_certificate", () -> mapping_cone_certificate(raw), 100_000)
end

function main(args = ARGS)
    run_self_check()
    if "--benchmark" in args
        run_benchmarks()
    else
        run_report()
    end
end

if abspath(PROGRAM_FILE) == @__FILE__
    main()
end
\end{lstlisting}

\section{Validation, System Metrics, and Operational Boundaries}

\subsection{Verification Protocols}
To ensure the DCSE-GBI co-processor is deployed safely, engineering teams must execute a three-tiered validation plan before writing back any transactions to a production EHR:

\begin{enumerate}
    \item \textbf{Mathematical Validation:}
    \begin{itemize}
        \item \emph{Assertion 1:} The Boolean-algebra implementation should be stress-tested over at least $2^{16}$ randomized join/meet/complement operations, including atom-disjointness and closure checks. This is an implementation test, not a proof of the algebraic specification.
        \item \emph{Assertion 2:} For a declared evidence box $E_{K,\epsilon,A}$, numerical validation should sweep corners and adversarial near-boundary cases and require $\lambda_{\min}(I(\alpha))$ to exceed an implementation tolerance such as $10^{-6}$. The tolerance is an engineering criterion, not a theorem that follows from a finite sample alone.
        \item \emph{Assertion 3:} When an actual mapping-cone Laplacian is constructed, the implementation should verify symmetry and positive semi-definiteness to numerical tolerance. Trace-based stalk energies should also be checked under randomized orthogonal basis rotations; the toy appendix demonstrates basis invariance for its projector surrogate, not a full mapping-cone implementation.
    \end{itemize}
    \item \textbf{Systems Validation:}
    \begin{itemize}
        \item \emph{Attestation Verification:} Deploy the system in a staging environment and inject invalid, expired, or modified remote attestation claims \cite{intel_sgx_attestation}. Verify that invalid or stale attestation causes the governed write path to fail closed. Key revocation should be tested only in deployments whose key-management design actually uses attestation-bound database keys.
        \item \emph{Consensus Fault Injector:} Simulate network partition and validator crash faults to test the liveness and safety invariants of the non-equivocating Byzantine state logs. For any selected BFT protocol, test the protocol-specific resilience and quorum conditions. Under the usual classical $3f+1$ model, tolerating $f$ Byzantine faults requires the active replica population to satisfy $n\ge3f+1$ \cite{pbft,bedrock_bft}; falling below the required threshold must halt authoritative writes.
        \item \emph{Rollback Conformance:} Inject a corrupted medication order nested within a multi-resource FHIR transaction bundle. Verify that the target FHIR gateway rejects the transaction atomically and returns an appropriate \texttt{OperationOutcome}; capture the failed attempt in the deployment's audit/provenance path according to policy. Unrelated resources outside the submitted transaction should remain unchanged.
    \end{itemize}
    \item \textbf{Clinical Validation:}
    \begin{itemize}
        \item \emph{Retrospective Playback:} A future clinical validation program should preregister a sufficiently powered retrospective or shadow-mode cohort, with expert adjudication and subgroup analysis, before any clinical deployment claim. A nominal target such as $10{,}000$ encounters may be useful for planning, but sample size must be justified by prevalence, target confidence intervals, and the intended claims rather than fixed by this manuscript.
    \end{itemize}
\end{enumerate}

\subsection{Performance and Reliability Metrics}
The following values are proposed engineering targets for validation and benchmarking; they are not presented here as externally certified thresholds or empirically established safety guarantees:

\begin{table}[htbp] 
\centering 
\caption{Proposed Operational Performance and Reliability Targets} 
\label{tab:metrics} 
\small 
\begin{tabularx}{\textwidth}{l l c X} 
\toprule 
\textbf{Metric Group} & \textbf{Specific Measure} & \textbf{\begin{tabular}[c]{@{}c@{}}Proposed\\ Baseline\end{tabular}} & \textbf{Testing Methodology} \\ 
\midrule 
Mathematical & Spectral Gap ($\lambda_1 - \lambda_0$) & $\ge 0.15$ & SVD spectrum sweep on $L_C$ \\ 
\addlinespace 
Mathematical & Fisher Matrix Condition Number & $\le 10^4$ & SVD evaluation on $E_{K,\epsilon,A}$ \\ 
\addlinespace 
Systems & End-to-End Latency (Enclave) & $\le 150 \text{ ms}$ & Benchmarking from FHIR request to response \\ 
\addlinespace 
Systems & Attestation Bootstrapping Time & $\le 2.5 \text{ s}$ & TEE/IAS handshake latency measurement \\ 
\addlinespace 
Clinical & Severe Contradiction Sensitivity & $100\%$ & Golden-standard retrospective chart injection \\ 
\addlinespace 
Clinical & False Conflict Adjudication Rate & $\le 4\%$ & Shadow-mode user experience trial \\ 
\bottomrule 
\end{tabularx} 
\end{table}

\subsection{Strategic Boundaries and Fail-Closed Constraints}
The DCSE-GBI framework is a fail-closed clinical co-processor, not an autonomous practitioner. It is designed to safely restrict data transitions rather than deduce patient care paths. Operational teams must strictly enforce the following execution limits:
\begin{itemize}
    \item \textbf{Boundary 1:} The co-processor must immediately halt and deny write permissions if patient identity resolution yields an ambiguous match, regardless of demographic similarity scores.
    \item \textbf{Boundary 2:} Unsigned, unpinned, or stale terminology updates must automatically trigger an administrative freeze, preventing clinical lookups from executing under obsolete mapping semantics.
    \item \textbf{Boundary 3:} The visual audit chart's linear distortion coefficient $K$ acts solely as a notification tool to assist manual human review. High visual distortion must never be interpreted as proof of logical inconsistency. A quarantine decision must follow the declared deterministic policy over validated identity, terminology, provenance, temporal, dependency, and consistency checks; mapping-cone-derived energies may be one input only when an actual cone construction has been implemented and validated.
\end{itemize}


\begin{thebibliography}{99}


\bibitem{scholze_condensed}
D. Clausen and P. Scholze,
\emph{Condensed Mathematics and Complex Geometry},
arXiv:2605.11731, 2026.
\url{https://arxiv.org/abs/2605.11731}.

\bibitem{scholze_lectures}
P. Scholze,
\emph{Lectures on Condensed Mathematics},
arXiv:2605.03658, 2026.
\url{https://arxiv.org/abs/2605.03658}.

\bibitem{reshetnyak_distortion}
Y. G. Reshetnyak,
\emph{Space Mappings with Bounded Distortion},
Translations of Mathematical Monographs,
American Mathematical Society, 1989.

\bibitem{vaisala}
J. V\"ais\"al\"a,
\emph{Lectures on $n$-Dimensional Quasiconformal Mappings},
Lecture Notes in Mathematics, vol. 229,
Springer, 1971.


\bibitem{hansen_ghrist_spectral}
J. Hansen and R. Ghrist,
``Toward a spectral theory of cellular sheaves,''
\emph{Journal of Applied and Computational Topology},
vol. 3, pp. 315--358, 2019.
doi:10.1007/s41468-019-00038-7.

\bibitem{weibel_homological}
C. A. Weibel,
\emph{An Introduction to Homological Algebra},
Cambridge University Press, 1994.


\bibitem{pbft}
M. Castro and B. Liskov,
``Practical Byzantine Fault Tolerance,''
in \emph{Proceedings of the 3rd Symposium on Operating
Systems Design and Implementation (OSDI '99)},
USENIX Association, pp. 173--186, 1999.
\url{https://www.usenix.org/conference/osdi-99/practical-byzantine-fault-tolerance}.

\bibitem{minbft}
G. S. Veronese, M. Correia, A. N. Bessani,
L. C. Lung, and P. Verissimo,
``Efficient Byzantine Fault-Tolerance,''
\emph{IEEE Transactions on Computers},
vol. 62, no. 1, pp. 16--30, 2013.
doi:10.1109/TC.2011.221.

\bibitem{bedrock_bft}
M. J. Amiri, C. Wu, D. Agrawal, A. El Abbadi,
B. T. Loo, and M. Sadoghi,
\emph{The Bedrock of Byzantine Fault Tolerance:
A Unified Platform for BFT Protocol Design and Implementation},
arXiv:2205.04534, 2022.
\url{https://arxiv.org/abs/2205.04534}.

\bibitem{zhao_bythos}
Q. Zhao, G. P\^{\i}rlea, K. Grzeszkiewicz,
S. Gilbert, and I. Sergey,
``Compositional Verification of Composite Byzantine Protocols,''
in \emph{Proceedings of the 2024 ACM SIGSAC Conference
on Computer and Communications Security (CCS '24)},
pp. 34--48, 2024.
doi:10.1145/3658644.3690355.

\bibitem{bythos_artifact}
Q. Zhao, G. P\^{\i}rlea, K. Grzeszkiewicz,
S. Gilbert, and I. Sergey,
\emph{Bythos: Compositional Verification of Composite
Byzantine Protocols---Research Artifact},
Zenodo, 2024.
doi:10.5281/zenodo.12787570.


\bibitem{intel_sgx}
Intel Corporation,
\emph{Intel Software Guard Extensions (Intel SGX)},
Developer Documentation and Overview,
accessed August 8, 2026.
\url{https://www.intel.com/content/www/us/en/developer/tools/software-guard-extensions/overview.html}.

\bibitem{intel_sgx_attestation}
Intel Corporation,
\emph{Attestation Services for Intel Software Guard Extensions},
accessed August 8, 2026.
\url{https://www.intel.com/content/www/us/en/developer/tools/software-guard-extensions/attestation-services.html}.

\bibitem{intel_sgx_pse}
Intel Corporation,
``What Is the Role of the Intel Software Guard Extensions
Platform Services Enclave (PSE) and How Is It Invoked?''
Intel Support Article 000058691,
reviewed August 6, 2021.
\url{https://www.intel.com/content/www/us/en/support/articles/000058691/software/intel-security-products.html}.


\bibitem{fhir_r4}
HL7 International,
\emph{FHIR Release 4: RESTful API}.
\url{https://hl7.org/fhir/R4/http.html}.

\bibitem{fhir_provenance}
HL7 International,
\emph{FHIR Release 4: Provenance Resource}.
\url{https://hl7.org/fhir/R4/provenance.html}.

\bibitem{fhir_outcome}
HL7 International,
\emph{FHIR Release 4: OperationOutcome Resource}.
\url{https://hl7.org/fhir/R4/operationoutcome.html}.

\bibitem{smart_launch}
HL7 International,
\emph{SMART App Launch Implementation Guide},
version 2.2.0.
\url{https://hl7.org/fhir/smart-app-launch/}.


\bibitem{hipaa_audit_controls}
U.S. Department of Health and Human Services,
``Audit Controls,''
45 C.F.R. \S 164.312(b) (2026).
\url{https://www.ecfr.gov/current/title-45/subtitle-A/subchapter-C/part-164/subpart-C/section-164.312}.


\bibitem{lee_safety_receipt_patent}
Y. B. Lee,
\emph{Safety Receipt Layer for Permit-Before-Action Gating
of AI-Assisted Clinical Decision Support Interventions},
U.S. Patent No. 12,633,414 B1,
May 19, 2026.


\bibitem{fda_amoxil}
U.S. Food and Drug Administration,
\emph{AMOXIL (Amoxicillin): Prescribing Information},
NDA 050542, Supplement 032, 2024.
\url{https://www.accessdata.fda.gov/drugsatfda_docs/label/2024/050542s032lbl.pdf}.

\bibitem{metformin_label}
U.S. National Library of Medicine,
\emph{DailyMed: Metformin Hydrochloride Tablets,
Full Prescribing Information},
current U.S. labeling, accessed August 8, 2026.
\url{https://dailymed.nlm.nih.gov/dailymed/drugInfo.cfm?setid=59a4ccda-5487-4ac4-9373-0c0fc8b7c88b}.


\bibitem{ihs_modernization}
Indian Health Service,
\emph{PATH EHR, IHS Health Information Technology Modernization Program},
accessed August 10, 2026.
\url{https://www.ihs.gov/HIT/path-ehr/}.

\bibitem{ihs_rpms}
Indian Health Service,
\emph{RPMS EHR Technical Overview},
accessed August 10, 2026.
\url{https://www.ihs.gov/ehr/technicaloverview/}.


\bibitem{boundarybench_repo}
A. Spivey,
\emph{GBI BoundaryBench v0.1: Programmatic Verification of Legacy-EHR Transformations},
public research showcase and aggregate results, 2026.
\url{https://github.com/AlvinSpivey/GBI-BoundaryBench}.


\bibitem{jordan_fokker}
R. Jordan, D. Kinderlehrer, and F. Otto,
``The variational formulation of the Fokker--Planck equation,''
\emph{SIAM Journal on Mathematical Analysis},
vol. 29, no. 1, pp. 1--17, 1998.

\bibitem{koopman}
B. O. Koopman,
``Hamiltonian systems and transformation in Hilbert space,''
\emph{Proceedings of the National Academy of Sciences},
vol. 17, no. 5, pp. 315--318, 1931.

\bibitem{ib}
N. Tishby, F. C. Pereira, and W. Bialek,
``The information bottleneck method,''
in \emph{Proceedings of the 37th Annual Allerton Conference
on Communication, Control, and Computing},
1999.

\bibitem{tda}
H. Edelsbrunner and J. Harer,
\emph{Computational Topology: An Introduction},
American Mathematical Society, 2010.


\bibitem{superposition}
N. Elhage et al.,
``Toy Models of Superposition,''
\emph{Transformer Circuits Thread}, Anthropic, 2022.

\bibitem{logitlens}
Nostalgebraist,
``Interpreting GPT: The Logit Lens,''
2020.

\end{thebibliography}
\end{document}